\documentclass[11pt, english]{article}
\usepackage[toc,titletoc,page]{appendix}
\usepackage[
backend=biber,
style=alphabetic,
maxbibnames=10,
maxalphanames=10,
backref=true,
]{biblatex}
\newcommand\blfootnotea[1]{%
  \begingroup
  \renewcommand\thefootnote{}\footnote{#1}%
  \endgroup
}
\newcommand\numberthis{\addtocounter{equation}{1}\tag{\theequation}} %

\newcommand{\twopartdef}[4]
{
\left\{
\begin{array}{ll}
#1 & \mbox{if } #2 \\
#3 & \mbox{if } #4
\end{array}
\right.
}

\newcommand{\R}{\mathbb{R}}

\newcommand{\viol}{v_{\mathsf{RoS}}}
\newcommand{\xt}{x_t}
\newcommand{\bt}{b_t}
\newcommand{\vt}{v_t}
\newcommand{\pt}{p_t}

\newcommand{\gt}{\gamma_t}

\newcommand{\reg}{\mathrm{Regret}}
\newcommand{\E}{\mathbb{E}}
\newcommand{\rew}{\mathrm{Reward}}
\newcommand{\opt}{\mathsf{Opt}}
\newcommand{\Prob}{\mathrm{Prob}}

\newcommand{\gseq}{\overrightarrow{\gamma}}

\newcommand{\err}{\mathrm{Err}}

\newcommand{\tcparatio}{\mathsf{RoS}}

\global\long\def\1{\mathbf{1}}%

\global\long\def\vt{v_{t}}%
\global\long\def\xt{x_{t}}%
\global\long\def\pt{p_{t}}%
\global\long\def\gt{g_t}
\global\long\def\bt{b_{t}}%
\global\long\def\mut{\mu_{t}}%
\global\long\def\0{\mathbf{0}}%
\global\long\def\fts{f_{t, \mathsf{RoS}}^{\star}}%
\global\long\def\ftscombined{f_{t, \mathsf{combined}}^{\star}}
\global\long\def\lamt{\lambda_{t}}%
\global\long\def\mub{\overline{\mu}_{\tau}}%
\global\long\def\lamb{\overline{\lambda}_{\tau}}%
\global\long\def\reg{\mathsf{Regret}}%
\global\long\def\rew{\mathsf{Reward}}%
\global\long\def\alg{\mathsf{Alg}}%
\global\long\def\err{\mathsf{Err}}%

\global\long\def\bard{\overline{\mathcal{D}}_{\mathsf{RoS}}}%
\global\long\def\bardcombined{\overline{\mathcal{D}}_{\mathsf{combined}}}%
\global\long\def\calp{\mathcal{P}}%
\global\long\def\E{\mathbb{E}}%
\global\long\def\fs{f^{\star}}%
\global\long\def\mup{\mu^{\prime}}%
\global\long\def\lamp{\lambda^{\prime}}%

\global\long\def\stm{\sigma_{t-1}}%

\usepackage[left= 1in, right= 1in, top = 1in, bottom = 1in]{geometry}
\usepackage{amsmath, amsthm, amssymb}
\usepackage{algorithm, algorithmicx, algpseudocode}
\usepackage{enumitem}
\usepackage{hyperref}
\usepackage[dvipsnames]{xcolor}

\hypersetup{
	colorlinks=true,
	pdfpagemode=UseNone,
	citecolor=ForestGreen,
	linkcolor=NavyBlue,
	urlcolor=magenta
}

\usepackage{thmtools} 
\usepackage{thm-restate}
\declaretheoremstyle[
  headfont=\bfseries,
  bodyfont=\normalfont\itshape,
]{mydefinition}
\theoremstyle{plain}
\newtheorem{thm}{\protect\theoremname}
\newtheorem{theorem}{Theorem} 
\theoremstyle{plain}
\newtheorem{lem}[thm]{\protect\lemmaname}
\theoremstyle{plain}

\theoremstyle{plain}

\theoremstyle{plain}
\newtheorem{assumption}[thm]{\protect\assumptionname}
\theoremstyle{remark}

\newtheorem{corollary}{Corollary}
\theoremstyle{mydefinition}
\newtheorem{definition}{Definition}
\newtheorem{proposition}{Proposition}
\providecommand{\lemmaname}{Lemma}
\providecommand{\remarkname}{Remark}
\providecommand{\assumptionname}{Assumption}

\numberwithin{equation}{section}
\numberwithin{theorem}{section}
\numberwithin{fact}{section}
\numberwithin{equation}{section}
\numberwithin{definition}{section}
\numberwithin{lem}{section}
\numberwithin{remark}{section}
\numberwithin{corollary}{section}
\numberwithin{assumption}{section}
\numberwithin{algorithm}{section}
\numberwithin{proposition}{section}

\usepackage[capitalize, nameinlink]{cleveref} 

\crefname{prob}{Problem}{Problems}
\crefname{fact}{Fact}{Facts}
\crefname{sec}{Section}{Sections}
\crefname{app}{Appendix}{Appendices}

\crefformat{none}{(#2#1#3)}
\crefname{equation}{Equation}{Equations}
\crefname{lem}{Lemma}{Lemmas}
\crefname{rem}{Remark}{Remarks}
\crefname{lemma}{Lemma}{Lemmas}
\crefname{defn}{Definition}{Definitions}
\crefname{cor}{Corollary}{Corollaries}
\crefname{prop}{Proposition}{Propositions}
\crefname{ineq}{Inequality}{Inequalities}
\crefname{alg}{Algorithm}{Algorithms}
\crefname{assumption}{Assumption}{Assumptions}
\creflabelformat{ineq}{#2{\upshape(#1)}#3} 
\crefalias{thm}{theorem}

\title{Online Bidding Algorithms for \\
Return-on-Spend Constrained Advertisers\blfootnotea{Author names are listed in alphabetical order.}}
\author{
Zhe Feng \\ Google Research \\ \texttt{zhef@google.com} 
\and 
Swati Padmanabhan\thanks{Work done as a student researcher in the Market Algorithms team at Google Research.} \\ University of Washington, Seattle \\ \texttt{pswati@uw.edu}
\and
Di Wang \\ Google Research \\ \texttt{wadi@google.com}
}

\date{}

\begin{document}
\maketitle

\begin{abstract}
Online advertising has recently grown into a highly competitive and complex multi-billion-dollar industry, with advertisers bidding for ad slots at large scales and high frequencies. This has resulted in a growing need for efficient ``auto-bidding'' algorithms that determine the bids for incoming queries to maximize advertisers' targets subject to their specified constraints. This work explores efficient online algorithms for a single value-maximizing advertiser under an increasingly popular constraint: Return-on-Spend (RoS). We quantify efficiency in terms of \emph{regret} relative to the optimal algorithm, which knows all  queries a priori. 

We contribute a simple online algorithm that achieves near-optimal regret in expectation while always respecting the specified RoS constraint when the input sequence of queries are i.i.d.\ samples from some distribution. We also integrate our results with the previous work of Balseiro, Lu, and Mirrokni~\cite{BLM20} to achieve near-optimal regret while respecting \emph{both} RoS and fixed budget constraints. 

Our algorithm follows the primal-dual framework and uses online mirror descent (OMD) for the dual updates. However, we need to use a non-canonical setup of OMD, and therefore the classic low-regret guarantee of OMD, which is for the adversarial setting in online learning, no longer holds. Nonetheless, in our case and more generally where low-regret dynamics are applied in algorithm design, the gradients encountered by OMD can be far from adversarial but influenced by our algorithmic choices. We exploit this key insight to show our OMD setup achieves low regret in the realm of our algorithm.

\end{abstract}

\newpage

\section{Introduction}\label{sec:intro}
With the explosive growth of online advertising into a billion-dollar industry\footnote{See \href{https://www.insiderintelligence.com/insights/programmatic-digital-display-ad-spending/}{eMarketer}}, auto-bidding --- the practice of using optimization algorithms to generate bids for ad slots on behalf of advertisers ---  has emerged as a  predominant tool in online advertising~\cite{aggarwal2019autobidding, balseiro2019learning, babaioff2021non, golrezaei2021bidding, deng2021towards,  balseiro2021robust, balseiro2021landscape}. Unlike manual CPC (``cost-per-click'') bidding, which requires advertisers to manually update bids for new search queries, auto-bidding requires advertisers to specify only their high-level objectives and constraints. The advertising platform then deploys its auto-bidding agent, which, based on its underlying optimization algorithms, transforms these inputs into fine-grained bids. Thus, designing an efficient bidding algorithm for advertisers to achieve their targets under constraints constitutes a central problem in the auto-bidding domain.

Throughout this paper, we focus on the return-on-spend (RoS) constrained bidding problem for a single learner (auto-bidder). In this problem, the RoS constraint requires that the ratio of total value of the advertiser to its total payment exceed at least some specified target ratio. In practice, the RoS constraint may capture other similar constraints, e.g., target cost-per-acquisition (tCPA) and target return-on-ad-spend (tROAS).\footnote{See \href{https://support.google.com/google-ads/answer/2979071?hl=en}{Google ads support page}} Additionally, we investigate the well-studied total budget constraint, which specifies an upper bound on the auto-bidder's total expenditure; our proposed algorithm, though tailored to the RoS constraint, easily adapts to the budget constraint, as well.

We study the online bidding algorithm for a single auto-bidder (learner) in the stochastic setting: In each round, an ad query (request) and auction are generated i.i.d.\ from an unknown distribution, after which the learner submits a bid to compete for the ad query in the auction. Given all bids for this query, the auction mechanism specifies which advertiser wins the opportunity to show its ad to the user and how much it needs to pay. Before placing its bid, the learner  observes only the value of this ad query; after placing the bid, it  receives its bid's auction outcome (i.e., the allocation and payment). From the perspective of auto-bidding, the platform needs to design an online bidding algorithm for the learner to maximize its target subject to the RoS and budget constraints. 

For theoretical simplicity, we focus only on the value-maximizing auto-bidder, i.e., the learner aims to maximize the total realized value (or conversions) for $T$ i.i.d.\ randomly drawn ad queries subject to RoS and budget constraints. Our results can be easily extended to handle other types of auto-bidders (e.g., utility maximization or a hybrid of value maximization and utility maximization~\cite{balseiro2021robust}).

\subsection{Results and Techniques}
Our main result is as follows. 
\begin{theorem}[Informal version; see \cref{lem:combinedRegretBothStrict}]
We provide an algorithm (\cref{alg:combined-truthful-strict}) for value maximization under RoS and budget constraints. For a $T$-length input i.i.d.\ sequence of ad queries, our algorithm provably attains $O(\sqrt{T}\log T)$ regret while respecting both the budget and RoS constraints. 
\end{theorem}
To the best of our knowledge,  ours is the first algorithm to attain near-optimal regret while satisfying \emph{both} budget and RoS constraints in any outcome. In doing so, we improve upon the prior work of \cite{BLM20}, which obtains similar guarantees under only budget constraints. Our result holds for i.i.d.\ input sequences and under an additional mild technical assumption on the input distribution.

\looseness=-1We build up to our result by first obtaining a $O(\sqrt{T})$-regret algorithm (\cref{alg:tcpa-truthful-soft}) for value maximization under \emph{approximate} RoS constraints (i.e., allowing for an at most $O(\sqrt{T}\log T)$ violation of this constraint). A simple modification to this algorithm yields  \cref{alg:meta-strictconstraints-tcpa}, which obtains a $O(\sqrt{T}\log T)$ regret guarantee under a \emph{strict} RoS constraint (\cref{thm:strictTcpaOnly}). We then combine \cref{alg:tcpa-truthful-soft} with ideas from \cite{BLM20} to design \cref{alg:combined-truthful-soft}, which attains a $O(\sqrt{T})$-regret under \emph{two} constraints: strict budget constraints and approximate  RoS constraints (with a $O(\sqrt{T}\log T)$ violation cap). Finally, unifying ideas from \cref{alg:meta-strictconstraints-tcpa} and  \cref{alg:combined-truthful-soft} yields our main result.  

Underlying all the algorithms we propose in this paper is a primal-dual framework similar to that used by~\cite{BLM20}. Such a framework allows our algorithms to adapt to the changing values and prices of input queries while respecting the advertisers' specified constraints and goals over the entire time horizon. In particular, the  dual variable (which tracks the constraint violation) is updated via online mirror descent (OMD) using the generalized cross-entropy function, which imposes a large (exponential) penalty on constraint violation.

An immediate technical challenge in attempting to use generalized cross-entropy as the mirror map is its lack of strong convexity on the non-negative orthant. While this is an issue in~\cite{BLM20} as well, in this case, the fixed budget constraint bounds the corresponding dual variable by a constant, which in turn implies the desired strong convexity on the space over which their algorithm operates. In our case (with the RoS constraint), there exist example inputs on which no such bound exists, which necessitates a novel analysis that circumvents the lack of strong convexity. 

Our key insight is that while the low-regret guarantee of OMD with a strongly convex mirror map holds even when the gradients seen by OMD are adversarial, when we use OMD in our algorithm to update the dual variable, the gradients used in the update are controlled by how our algorithm sets the primal variables (i.e., bids); this connection is sufficient to give our algorithm the low-regret guarantee despite not using a strongly convex mirror map. We use this connection in a white-box adaptation of the OMD analysis tailored to our algorithm, as well as  properties of the generalized cross-entropy function from  the (offline) positive linear programming literature~\cite{allen2014using}. 

As a final remark on our technique, to turn an algorithm that only \emph{approximately} satisfies a constraint to one that \emph{strictly} satisfies it, we propose a simple strategy: We first submit a sequence of bids that lets us accumulate a slack on the RoS constraint, followed by the existing algorithm, which suffers some bounded constraint violation (c.f.~\cref{sec:hard_constraint}). The first phase builds up enough slack (at the cost of bidding sub-optimally) to compensate for the constraint violation from the later iterations. This allows us to trade off the violation on the RoS constraint with the objective value. We anticipate  that this simple idea might be applicable in other contexts but note that such a trade-off is typically easier in offline optimization via retrospectively modifying the solution (e.g., by scaling or truncating).  

\subsection{Related Work}
Our problem falls under the broader umbrella of online optimization under stochastic time-varying constraints, and it has seen a long line of research by various research communities,  e.g.~\cite{mannor2009online,mahdavi2012trading, mahdavi2013stochastic, yu2017online, yu2020low,BLM20,agrawal2014fast,castiglioni2022unifying,badanidiyuru2018bandits, immorlica2019adversarial}. From a technical standpoint, our primal-dual approach is similar to that in~\cite{BLM20}; however, the RoS constraint is not a packing-type constraint as studied in their work as well as in~\cite{badanidiyuru2018bandits, immorlica2019adversarial}. There also exist papers that study a variant of our problem with a constraint class that contains as a special case our RoS constraint (e.g.~\cite{agrawal2014fast,castiglioni2022unifying}); however, these general techniques do not provide guarantees as strong as ours, as we elaborate next. 

For example, a recent work \cite{castiglioni2022unifying} gives a primal-dual framework using regret minimization, which, when adapted to our bidding problem under the RoS constraint, achieves $\tilde{O}(T^{3/4})$ regret with $\tilde{O}(T^{3/4})$ constraint violation (both with high probability). Both bounds are polynomially weaker than our guarantees (that further hold deterministically). Their bounds can be improved to $\tilde{O}(T^{1/2})$ under a `strictly feasible' assumption,  which is essentially~\cref{assumption:beta} in our context; in contrast, we guarantee \textit{strict} constraint satisfaction under that assumption. Moreover, their algorithm uses techniques from the multi-armed bandits literature, thus requiring the space of values and bids to be of finite size $n_v,n_b$ respectively, and their regret bound scales with $n_v\sqrt{n_b}$. Our algorithm works directly with continuous values and bids. 

Another example is~\cite{agrawal2014fast}, which considers general online optimization with convex constraints. This work uses black-box low-regret methods that rely on a globally strongly-convex regularizer over the dual space, and a sub-linear regret bound is attainable only when the dual space is well-bounded (e.g. a scaled simplex) or the dual variable can be projected onto such a space without incurring too much additional regret. This canonical approach turns out to be difficult for the RoS constraint, which can incur poor problem-specific parameters in generic guarantees. As a result, this technique cannot give sub-linear regret for the RoS constraint. To circumvent this issue, we rely on  problem-specific structure rather than globally strongly-convex regularization.

In the more specific domain of online bidding under constraints, a closely related work is~\cite{golrezaei2021bidding} which investigates the same problem we do but with the RoS and budget constraints holding \emph{only in expectation over the distribution}; in contrast, our constraint guarantees hold for any realization of samples. We note that \cite{golrezaei2021bidding} give an example where bidding based on the optimal offline (fixed) dual variable cannot achieve sub-linear regret. This does not contradict our results since our algorithm is adaptive rather than trying to converge to the optimal offline dual variables and then bid based on fixed dual variables afterwards. Our algorithm keeps updating our dual variables based on the previous outcomes, and this takes advantage of stochastic information to balance the objective and constraints violation. 

The problem  of \emph{learning to bid in repeated auctions} has been widely studied in both academia and industry, e.g.~\cite{borgs2007dynamics, weed2016online,Feng18,BFG21,Nedelec22,Noti21,HZFOW20}. These papers mainly abstract the problem of learning to bid as contextual bandits but do not incorporate constraints into them. Beyond this, there have been some work on bidding under budget constraints, e.g., \cite{balseiro2019learning, AWLZHD22}, however, these papers focus on utility-maximizing agents with at most one constraint. \cite{CCRKS22} also considers multiple different constraints in online bidding algorithms, however, they directly add a regularizer of one non-packing constraint in the objective and apply the standard dual mirror descent approach to design the algorithms. Their regret bound is measured against this relaxed objective, whereas ours is relative to the adaptive optimal benchmark.

Finally, loosely related work includes the AdWords problem~\cite{mehta2007adwords, devanur2009adwords}, which focuses  on budget management for multiple bidders to maximize the seller's revenue; in contrast, we focus on the design of online bidding algorithms for a single auto-bidder with RoS and budget constraints.

\section{Preliminaries}\label{sec:prelimsNotation}

We consider an online bidding model for a single learner (auto-bidder): At each time step $t$, nature stochastically generates for the learner an ad query associated with a value $v_t \in [0, 1]$  and an auction mechanism $(x_t, p_t)$, where $x_t: \R_{\geq 0} \mapsto [0, 1]$ is an allocation function and $\pt: \R_{\geq 0} \mapsto [0, 1]$ the expected payment rule.\footnote{Our algorithm's regret guarantee  depends linearly on the scales of the valuation and payment, and the assumed bounds on these quantities are for theoretical simplicity.} We assume the following stochastic model: $(v_t, x_t, p_t)$ are drawn independently and identically (i.i.d.) from an unknown distribution. At each time step $t$, the value $\vt$ is known to the learner before making a bid, and the learner decides its bid $b_t$ given $v_t$ and historical information. At the end of time step $t$, the learner observes the realized outcome from the auction mechanism, i.e., $x_t(b_t)$ and $p_t(b_t)$.

We focus only on auctions that are truthful. This requires that the allocation function $x_t(b)$ be non-decreasing with the input bid $b$ and the payment function $p_t$ be characterized by the pioneering work of~\cite{myerson1981optimal} as
\[p_t(b) = b\cdot \xt(b) - \int_{z= 0}^b \xt(z)dz.\numberthis\label{eq:MyersonsLemma}\]
For instance, the well-known second-price auction for a single item is truthful, and its payment function satisfies \cref{eq:MyersonsLemma}. Note the payment must be zero when the allocation is zero, and the payment is also at most the bid. This work also assumes $v_t\cdot x_t(b_t)$ to be the realized value of the learner in each round.

We design online bidding algorithms to maximize the learner's total realized value subject to an RoS constraint. Formally, the optimization problem under RoS constraint we study is 
\[
\begin{array}{ll}
\underset{b_{t}: t=1, \cdots, T}{\mbox{maximize}} & \sum_{t=1}^T v_{t}\cdot x_{t}(b_t)\\
\mbox{subject to } &  \tcparatio \cdot \sum_{t=1}^T p_{t}(b_t) \leq  \sum_{t=1}^T v_{t}\cdot x_{t}(b_t), 
\end{array} \numberthis\label[prob]{eq:tcpa-obj-simple}
\] 
where $\tcparatio > 0$ is the target ratio of the RoS bidder. Throughout the paper we assume without loss of generality\footnote{For any $\tcparatio \neq 1$, we can scale the values to be $v_t := \tcparatio \cdot v_t$.} that $\tcparatio = 1$. As noted in \cref{sec:intro}, our results can be extended to handle different learner objectives, e.g., a hybrid version between utility maximizing and value maximizing $\sum_{t=1}^T v_t \cdot x_t(b_t) - \tau p_t(b_t)$ for some $\tau\in[0,1]$.
 
To simplify the notation, we denote the difference between value and price in iteration $t$ as  
 \[g_t(b) := \vt \cdot \xt(b) - \pt(b),\numberthis\label{eq:defGradTcpa}
 \] 
and using this notation, the RoS constraint in \cref{eq:tcpa-obj-simple} may be stated as 
\[
\sum_{t = 1}^T g_t(\bt) \geq 0.\numberthis\label[ineq]{tcpaconstraint}
\]
Our algorithm also extends to the bidding problem subject to an additional budget constraint. 
\[ 
\begin{array}{ll}
     \underset{b_{t}: t=1, \cdots, T}{\mbox{maximize}} & \vt\cdot \xt(\bt)  \\
     \mbox{subject to}& \sum_{t=1}^T \pt(\bt) \leq \sum_{t=1}^T \vt \cdot \xt(\bt), \\
     &\sum_{t = 1}^T \pt(\bt) \leq \rho T. 
\end{array}\numberthis\label[prob]{combinedOptnProb}
\]
where $\rho T$ is the budget and $\rho > 0$ (assumed a fixed constant) measures the limit of the average expenditure over $T$ rounds (ad queries).

To collect notation, we denote the sample (ad query and auction) at time $t$ as a tuple $\gamma_t = (v_t, p_t, x_t)$ and assume $\gamma_t \sim \calp$ for all $t\in [T]$. We denote the sequence of $T$ samples by $\gseq := \{\gamma_1, \gamma_2, \dots, \gamma_T\} \sim \calp^T$ and sequences of length $\ell\neq T$ by $\gseq_{\!\ell}$ where needed. 

\paragraph{Analysis setup.} We use the notions of regret and constraint violation to measure the performance of our algorithms. To define the regret, we first define the reward of $\alg$ for a sequence of requests $\gseq$ over a time horizon $T$ as \[ \rew(\alg, \gseq):=\sum_{t = 1}^T \vt\cdot \xt(\bt). \numberthis\label{eq:defReward}\] Next, we define the optimal value in the same setup as for $\alg$ as \[\rew(\opt,\gseq) := \mbox{maximum}_{b_t\in \mathcal{B}} \sum_{t=1}^T \vt\cdot \xt(\bt),\numberthis\label{eq:defOPTtcpa} \] where $\mathcal{B}$ is the exact set of constraints. These definitions lead to the definition of regret of $\alg$ in this setup as \[\reg(\alg, \calp^T):= \E_{\gseq\sim \calp^T} \left[ \rew(\opt,\gseq) - \rew(\alg, \gseq)\right].\numberthis\label{eq:defRegret}\] We remark that we define $\rew$ for some specific input sequence, whereas $\reg$ is defined with respect to a distribution. 

 Finally, we use online mirror descent as a technical component in our analysis. In this regard, we use $V_h(y, x)=h(y)-h(x)-h^\prime(x)\cdot (y-x)$ to denote the Bregman divergence of $y$ in reference to $x$, measuring with the distance-generating function (``mirror map'') $h$. We brief review online mirror descent in \cref{sec:omd}. 
 
\section{Bidding Under an Approximate RoS Constraint}\label{sec:softTCPAconstraint}

In this section, we design and analyze an algorithm for \cref{eq:tcpa-obj-simple} allowing for bounded \emph{sublinear violation} of the RoS constraint. To this end, we first rewrite \cref{eq:tcpa-obj-simple} as the equivalent problem
\[\mbox{maximize}_{\{b_i\}} \left\{ \sum_{i=1}^T v_i \cdot x_i(b_i) +  \min_{\lambda\geq 0} \lambda \cdot \sum_{i=1}^T \left[ v_i\cdot x_i (b_i) - p_i(b_i) \right]\right\},
\numberthis\label[prob]{eq:firstSoftTcpa} 
\] in which the inner minimization  strictly enforces $\sum_{i=1}^T \left[ v_i\cdot x_i (b_i) - p_i(b_i) \right] \geq 0$ via the variable $\lambda$ that applies an unbounded penalty for the  constraint  violation. Our algorithm iteratively updates, in each iteration $t$, the bid $b_t$ and a dual variable $\lambda_t$ that captures the current best penalizing (dual) variable $\lambda$, described shortly. 
  
\paragraph{Updating the bid.}  Based on the formulation in \cref{eq:firstSoftTcpa}, our algorithm  chooses the bid $b_t$ as the maximizer of the penalty-adjusted reward of the \emph{current round}, with the penalty applied by the current dual variable $\lambda_{t}$: \[
        \bt = \arg\max_{b\geq 0} \left[ \frac{1+\lambda_t}{\lambda_t} \cdot v_t\cdot \xt(b) - \pt(b)\right]= \vt + \frac{\vt}{\lambda_t},\numberthis\label{eq:bid-truthful-tcpa-soft}\] where the final step is because of the truthful nature of the auction.\footnote{Because it is a truthful auction, we have $\nabla_b x_t(b) \geq 0$, which, when used in the definition of $b_t$, gives the claimed final step.} The final expression for $b_t$ is consistent with the setting that we first observe only the value $v_t$ before making the bid. 

\paragraph{Updating the dual variable.}To maintain a meaningful dual variable, we relax the penalty on the constraint violation in \cref{eq:firstSoftTcpa} by  adding
a scaled regularization function $h(\lambda)$. This regularizer prevents $\lambda$ from getting too large: 
  \[\mbox{maximize}_{\{b_i\}} \left\{ \sum_{i=1}^T v_i \cdot x_i(b_i) +  \min_{\lambda\geq 0} \left[\lambda  \cdot \sum_{i=1}^T \left[ v_i\cdot x_i (b_i) - p_i(b_i) \right] + \alpha^{-1}h(\lambda) \right]\right\},\numberthis\label[prob]{eq:metaLambdaUpdateRule}\] where $\alpha>0$ is the scaling factor of the regularizer to be set later. At any iteration $t$, the value of $\lambda_{t+1}$ is chosen to be the minimizer of the inner constrained minimization problem (until iteration $i = t$). While there exist a number of valid regularization functions, in this paper, we choose  the generalized negative entropy $h(u) = u\log u - u$, which gives the following expression for $\lambda_{t+1}$. 
\[\lambda_{t+1} = \arg\min_{\lambda \geq 0} \left[ g_t(\bt) \cdot \lambda + \frac{1}{\alpha} V_h(\lambda, \lambda_{t})\right] = \lambda_{1} \cdot \exp\left[- \sum_{i = 1}^t \alpha\cdot  g_i(b_i)\right].\numberthis\label{eq:dualVarUpdateTCPA}
 \] Through this rule, a net constraint violation (i.e., $\sum_{i=1}^t g_i (b_i) \leq 0$) makes the next dual variable exponential in the net violation, which in turn shrinks the next bid (in \cref{eq:bid-truthful-tcpa-soft}); on the other hand, an accumulated buffer in the net constraint violation (i.e., $\sum_{i = 1}^t g_i (b_i) > 0$) encourages the next dual variable to be small, allowing the next bid to grow. We demonstrate in \cref{sec:squaredMirrorMap} that the use of another valid mirror map on this domain, $h(u) = \frac{1}{2}u^2$, does not provide a sufficiently strong regret or constraint violation guarantee. Intuitively, the reason for this is that the squared penalty is simply not as strong as the exponential.  
 
The  final algorithm with these choices of $b_t$ and $\lambda_t$ is stated in \cref{alg:tcpa-truthful-soft}.

\begin{algorithm}[h!]
\caption{Bidding algorithm of the RoS agent operating under approximate constraints in truthful auctions (i.i.d.\ inputs), with mirror map $h(u) = u \log u - u$.}\label[alg]{alg:tcpa-truthful-soft}

\begin{algorithmic}[1]
\State \textbf{Input:} Total time horizon $T$ and requests $\gseq$ from the distribution $\calp^T$. 
\State \textbf{Initialize:} Initial dual variable $\lambda_1=1$ and dual mirror descent step size $\alpha = \frac{1}{\sqrt{T}}$.

\For{$t=1,2,\cdots, T$}
    \State Observe the value $v_t$, and set the bid $b_t = \frac{1+  \lambda_{t}}{\lambda_{t}} v_t$. \label{line:tcpaAlgBid}

    \State Observe the price $p_t$ and allocation $x_t$ at  $b_t$, and compute $ g_t(\bt) = v_t \cdot x_t({b}_t) - p_t(b_t)$. \label{line:tcpaAlgGrad}

    \State Update the dual variable as $\lambda_{t+1} = \lambda_t \exp\left[-\alpha g_t(\bt)\right]$.
    \label{line:tcpaAlgLambda}
    
\EndFor
\State \Return The sequence $\{b_t\}_{t=1}^T$ of generated bids. 
\end{algorithmic} 
\end{algorithm}

\subsection{Analysis of \cref{alg:tcpa-truthful-soft}}
The main export of this section is the following theorem of \cref{alg:tcpa-truthful-soft} for \cref{eq:tcpa-obj-simple}. The proof of the theorem is a straightforward application of our chief technical results \cref{lem:constraint_violation} and \cref{lem:regret_bound} (with $\lambda=1/T$), which provide guarantees on the constraint violation and regret, respectively. We focus our discussion on these two lemmas, deferring the proof of the theorem to \cref{sec:pureTCPAappendix}. 
\begin{restatable}{theorem}{thmMainTCPA}\label{thm:MainThmTcpa}
 With i.i.d.\ inputs from a distribution $\calp$ over a time horizon $T$, the regret of \cref{alg:tcpa-truthful-soft} on \cref{eq:tcpa-obj-simple} is bounded by
\[
\reg(\cref{alg:tcpa-truthful-soft}, \mathcal{P}^T)\leq   O(\sqrt{T}). \numberthis\label[ineq]{eq:MainRegretBoundtCPA}
\] Further, the violation of the RoS constraint is at most $2\sqrt{T}\log T$. 
\end{restatable}

\subsection{Bound On Constraint Violation of \cref{alg:tcpa-truthful-soft}}
To conclude that the constraint described by \cref{tcpaconstraint} is violated by only a small amount in \cref{alg:tcpa-truthful-soft}, we observe that when the cumulative violation is (non-trivially) larger than $1/\alpha$, the exponential function  quickly makes $\lambda_t$ huge; in turn, our bid $b_t=v_t+\frac{v_t}{\lambda_t}$ prevents us from over-bidding. Formally, we show the following result, later used in \cref{thm:MainThmTcpa} to obtain the stated constraint violation bound. (The bound could possibly be sharpened to get rid of the $\log$.)

\begin{lem}
\label{lem:constraint_violation}
Consider the sequence 
$\{\lambda_t\}_{t=1}^T$ starting at $\lambda_1 = 1$ and evolving as $\lambda_{t+1} = \lambda_t \exp\left[-\alpha g_t(\bt)\right]$ where $g_t(\bt)$ satisfies $g_t(\bt) \geq \max\left(-1,-\frac{1}{\lambda_t}\right)$ and $\alpha =\frac{1}{\sqrt{T}}$. Then, \[-\sum_{t = 1}^T g_t(\bt) \leq 2\sqrt{T}\log T.\] 
\end{lem}
\begin{proof}
Based on the update rule for $\lambda_t$, we know $\lambda_{t+1} =\exp\left[-\alpha\sum_{t'=1}^{t}g_{t'}(b_{t'})\right]$. If $-\sum_{t = 1}^T g_t(\bt) \leq \sqrt{T}\log T$, we are done. If this is not the case, let $T^{\prime}$ be the last time that $-\sum_{t = 1}^{T^{\prime}} g_t(\bt) \leq \sqrt{T}\log T$, so we know for any $t>T^{\prime}+1$, the dual variable $\lambda_t$ must be larger than $T$ since 
\[
\lambda_t =\exp\left[-\alpha\sum_{t'=1}^{t-1}g_{t'}(b_{t'})\right]>\exp\left[\alpha\sqrt{T}\log T\right] = T \text{ for $t>T^\prime + 1$,}
\]
which suggests 
\begin{equation}
\label[ineq]{eq:small_violation}
-g_t(\bt)\leq \frac{1}{\lambda_t}\leq \frac{1}{T} \qquad \forall t>T^{\prime}+1.
\end{equation}
 Thus, 
\[ -\sum_{t=1}^T g_t(b_t) = \sum_{t=1}^{T^{\prime}} -g_t(b_t) + \sum_{t > T^{\prime}} -g_t(b_t) \leq \sqrt{T}\log T + 2\leq 2\sqrt{T}\log T, \] where the second step used \cref{eq:small_violation} to bound the terms after iteration $T'$ and the fact that there are at most $T$ such iterations.
\end{proof}

\subsection{Bound On Regret of \cref{alg:tcpa-truthful-soft}}
We first state the following technical upper bound on the regret, and our effort in the rest of the section is devoted to bounding the right-hand side of this result. The proof of the following result (provided in \cref{sec:pureTCPAappendix}) follows the primal-dual framework in the proof of Theorem 1 in \cite{BLM20}. 

\begin{restatable}{proposition}{lemWeakDuality}\label[prop]{lem:regretBoundLambdatGt} 
With i.i.d.\ inputs from a distribution $\calp$ over a time horizon $T$, the regret of \cref{alg:tcpa-truthful-soft} on \cref{eq:tcpa-obj-simple} is bounded by \[\reg(\cref{alg:tcpa-truthful-soft}, \mathcal{P}^T) \leq \mathbb{E}_{\gseq \sim \mathcal{P}^T} \left[\sum_{t\in [T]} \lambda_t\cdot g_t(\bt)\right],\] where  $g_t$ and $\lambda_t$ are as defined in \cref{line:tcpaAlgGrad} and \cref{line:tcpaAlgLambda} of \cref{alg:tcpa-truthful-soft}. We note that the bound on the right-hand side can be negative since \cref{alg:tcpa-truthful-soft} does not guarantee $\sum_{t = 1}^T g_t(\bt) \geq 0$, but we do show a bound on the worst-case constraint violation in  \cref{lem:constraint_violation}.
\end{restatable}
Computing a bound on the regret then requires bounding $\sum_{t=1}^T \lambda_t \cdot g_t(\bt)$, which we do in the following lemma (and this is where a major chunk of our technical contribution lies).

\begin{lem}
\label{lem:regret_bound}
For any input sequence $\gseq$ of length $T$, running \cref{alg:tcpa-truthful-soft} on \cref{eq:tcpa-obj-simple} generates sequences $\{b_t\}_{t=1}^T$, $\{g_t\}_{t=1}^T$ and $\{\lambda_t\}_{t = 1}^T$ such that  \[\sum_{t=1}^T g_t(b_t) \cdot \lambda_t  \leq  O(\sqrt{T}). \] 
\end{lem}
\begin{proof}
In order to bound the quantity $\sum_{t\in [T]}  g_t(b_t) \cdot \lambda_t$, we first write 
\begin{align*}
    \alpha g_t (\bt)\cdot \lambda_t &= \alpha g_t(\bt)\cdot (\lambda_t - \lambda_{t+1}) + \alpha g_t(\bt)\cdot \lambda_{t+1}. \numberthis\label{eq:regretBoundInt2}
\end{align*} From \cref{line:tcpaAlgLambda} in \cref{alg:tcpa-truthful-soft}, we have
\[ g_t(\bt) = {\alpha}^{-1} \log (\lambda_t/\lambda_{t+1}). \numberthis\label{eq:fooptimality}\]  For $h(u) = u \log u - u$,  the Bregman divergence $V_h(y, x) = h(y)-h(x) - h^\prime(x)\cdot (y-x)$ is 
\[V_h(y, x) = y \log (y/x) - y + x.\numberthis\label{eq:BregmanDivGenNegEnt}\] Applying \cref{eq:BregmanDivGenNegEnt} and \cref{eq:fooptimality} to \cref{eq:regretBoundInt2} gives
\begin{align*}
        \alpha g_t(\bt)\cdot \lambda_t &= \alpha g_t(\bt)\cdot (\lambda_t - \lambda_{t+1})  + \lambda_{t+1} \cdot \log(\lambda_t/\lambda_{t+1}) 
        \numberthis\label{eq:regretBoundInt23}\\
        &= \left[(1+\alpha g_t(\bt)) \cdot (\lambda_t - \lambda_{t+1})\right] - V_h(\lambda_{t+1}, \lambda_t)\\ 
        &\leq \left[(1+\alpha g_t(\bt)) \cdot (\lambda_t - \lambda_{t+1})\right] -  \left\{\frac{(\lambda_t - \lambda_{t+1})^2}{2 \max(\lambda_t, \lambda_{t+1})} \right\}\\
        &\leq  \left[\lambda_t - \lambda_{t+1}\right]  + \frac{1}{2}\alpha^2 g_t(\bt)^2\cdot \max(\lambda_t, \lambda_{t+1}), \numberthis\label[ineq]{eq:initBoundGtLambdaT}
 \end{align*} where the third step follows from the local strong convexity of $V_h$ as shown in \cite{allen2014using} (for completeness, we prove this fact in \cref{sec:omd}), and the final step is by Cauchy-Schwarz inequality. We now bound the term  $\frac{1}{2}\alpha^2 g_t(\bt)^2\cdot \max(\lambda_t, \lambda_{t+1})$ in a case-wise manner. 

 \begin{enumerate}[label=C.\arabic*]
\item Assume $g_t(\bt)\geq 0$. Then, the inequality $g_t(\bt) \leq 1$ (from \cref{prop:gradientProperties}) and our choice of $\alpha \leq \frac{1}{\sqrt{T}}$ imply $\alpha g_t (\bt)\leq \frac{1}{\sqrt{T}}$, which in turn implies   \[\lambda_{t+1} = \lambda_t \exp\left[-\alpha g_t(\bt)\right] \leq \lambda_t \cdot(1- \alpha g_t(\bt)/2) = \lambda_t - \frac{1}{2}\alpha \lambda_t\cdot g_t(\bt),\numberthis\label[ineq]{eq:case2-ineq1}\] where we used $\exp(-x) \leq 1-x/2$ for $x\in [0, 1.5]$. Finally, we use $0\leq g_t(\bt)\leq 1$ and \cref{eq:case2-ineq1} to obtain: 
\[
\frac{1}{2}\alpha^2 g_t(\bt)^2 \max(\lambda_t, \lambda_{t+1}) = \frac{1}{2}\alpha^2 g_t(\bt)^2 \lambda_t =  \alpha  \cdot \frac{\alpha g_t(\bt) \lambda_t}{2} \leq  \alpha(\lambda_t-\lambda_{t+1}). \numberthis\label[ineq]{eq:case2-ineq4}    
\]

\item Assume $g_t(\bt)<0$. Then by invoking \cref{prop:gradientProperties}, we have $0\geq\gt(\bt)\geq \max(-1,-1/\lambda_t) $. This gives  \[ \alpha^2 g_t(\bt)^2 \max(\lambda_t, \lambda_{t+1}) = \alpha^2 g_t(\bt)^2 \lambda_{t+1} \leq \alpha^2 \cdot \frac{1}{\lambda_t}\cdot \lambda_{t+1}.\numberthis\label[ineq]{eq:case2-ineq2}\] 
    
Since $g_t(\bt)\geq - 1$ and $\alpha = \frac{1}{\sqrt{T}}$, we have $-\alpha g_t(\bt) \leq 1$. This  implies 
\[ \lambda_{t+1} =\lambda_t \exp\left[-\alpha g_t(\bt)\right]\leq e\lambda_t.\numberthis\label[ineq]{eq:case2-ineq6}\] Plugging \cref{eq:case2-ineq6} back into \cref{eq:case2-ineq2} gives \[\frac{1}{2}\alpha^2 g_t(\bt)^2 \max(\lambda_t, \lambda_{t+1}) \leq 2\alpha^2.\numberthis\label[ineq]{eq:firstBoundNegG}\] Using $-\alpha g_t(b_t) \leq 1$ again allows us to claim $\lambda_{t+1} = \lambda_t \exp\left[-\alpha g_t(b_t)\right] \leq \lambda_t (1-2\alpha g_t(b_t))$, where we used $\exp(x) \leq 1+2x$ for $x\in [0, 1]$. Again applying $\lambda_t g_t(b_t) \geq -1$ gives \[
    \lambda_{t+1} - \lambda_{t} \leq  2\alpha. \numberthis\label[ineq]{eq:boundForGapT}
    \]
 \end{enumerate}
Thus, in all cases, we have \[
\frac{1}{2}\alpha^2 g_t(\bt)^2 \max(\lambda_t, \lambda_{t+1})\leq \alpha(\lambda_t-\lambda_{t+1}) +4\alpha^2. 
\]
Plug this back in our previous bound in \cref{eq:initBoundGtLambdaT}, we get
\begin{align*}
   \alpha g_t(\bt)\cdot \lambda_t 
   &\leq 2 (\lambda_t-\lambda_{t+1})+4\alpha^2.
\end{align*}
 Dividing throughout by $\alpha$ and summing over $t = 1, 2, \dots, T$, enables telescoping; plugging in the chosen values of $\lambda_1$ and $\alpha$ from \cref{alg:tcpa-truthful-soft}  then yields the claimed bound. 
\end{proof}

\section{Bidding Under a Strict RoS Constraint}
\label{sec:hard_constraint}
In this section, we introduce a simple technique that turns an algorithm with non-zero (but bounded) violation of the RoS constraint into one \emph{strictly} obeying it. As noted in \cref{sec:CombinedConstraints}, this technique extends to the problem with \emph{both} RoS and budget constraints. Our idea is as follows. 

Suppose we have an algorithm (say, $\mathsf{Alg}$), which can guarantee an at most $\viol$ violation of the RoS constraint on any sequence $\gseq$ of input requests. We start by bidding the true value $b_t=v_t$ in the initial iterations $t=1,\ldots,K(\gseq)$ for some $K(\gseq)$ --- we call this sequence of iterations the \emph{first phase}. In a truthful auction, this choice of bids guarantees  $g_t(\bt)=v_t\cdot x_t(b_t)-p_t(b_t) \geq 0$ for all $t\leq K(\gseq)$. In other words, the bidder builds up a buffer on the RoS constraint. We continue until the cumulative buffer $\sum_{t=1}^{K(\gseq)} g_t(\bt)$ increases to at least $\viol$. Starting at iteration $K(\gseq)+1$, we run $\mathsf{Alg}$ afresh (i.e. without accounting for the first phase); recall, this violates the RoS constraint by at most $\viol$ over the remaining iterations. We refer to this run of $\mathsf{Alg}$ as the \emph{second phase}. Since the buffer from the first phase is enough to offset $\mathsf{Alg}$'s violation in the second phase, there is no violation of the RoS constraint at the end. We display this idea in \cref{alg:meta-strictconstraints-tcpa}. 

\begin{algorithm}[h!]
\caption{Bidding algorithm of the RoS agent operating under strict constraints in truthful auctions (i.i.d.\ inputs)}\label{alg:meta-strictconstraints-tcpa}
\begin{algorithmic}[1]
\State \textbf{Input:} Total time horizon $T$ and requests $\gseq$ from the data distribution $\calp^T$. 

\State \textbf{Initialize:} Set $\viol = 2\sqrt{T}\log T$ and $t = 1$.

\While{$\sum_{i=1}^{t-1} g_i (b_i) \leq \viol$} \Comment{First phase}
    \State Observe the value $v_t$, and set the bid $b_t = v_t$.\label{line:MetaStrictAlgTCPABid}
    \State Observe the price $p_t$ and the allocation $x_t$ at $\bt$, and compute $g_t(\bt) := \vt \cdot \xt(\bt) - \pt(\bt)$. 
    \State Increment the iteration count $t= t+1$.
\EndWhile

\State Run \cref{alg:tcpa-truthful-soft} with time horizon $T-t$ and the remaining $T-t$ requests from $\gseq$ as input. \Comment{Second phase}

\State \Return The sequence $\{\bt\}_{t=1}^T$ of generated bids from both phases. 

\end{algorithmic}
\end{algorithm}

\subsection{Analysis of \cref{alg:meta-strictconstraints-tcpa}}
The high-level idea to guarantee a low regret for \cref{alg:meta-strictconstraints-tcpa} is to start with the observation that the reward collected by \cref{alg:meta-strictconstraints-tcpa} for any $\gseq$ in $T$ steps is at least that collected by \cref{alg:tcpa-truthful-soft} in the second phase (i.e., the last $T-K(\gseq)$ steps). The second phase suffers (in expectation) a regret bounded by the guarantee of \cref{thm:MainThmTcpa}. We then use the i.i.d.\ assumption on the input sequence to bound the gap between the expected reward collected by $\opt$ in a sequence of length $T-K(\gseq)$ to that in a sequence of length $T$, which naturally depends on the expected length of $K(\gseq)$; finally, we show this expected length is at most $O(\sqrt{T}\log T)$ under a mild technical assumption on the input distribution; this bounds the additional regret accrued over the first phase and completes the analysis. 

To formally see this, we need two simple technical tools (proved in \cref{sec:strictTCPA}) as follows. First, we make the following assumption on the distribution $\mathcal{P}$ and state a corresponding result.

\begin{assumption} \label{assumption:beta}Define the parameter $\beta$ of a distribution $\mathcal{P}$ as follows
\[\beta = \E_{v,p,x\sim \mathcal{P}} \left[\max(0,v\cdot x(v)-p(v)\right].
\]
We assume in our problem $\beta$ is an absolute constant bounded away from $0$ and independent of $T$.
\end{assumption}
The parameter $\beta$ is the expected amount of buffer we accrue per iteration during the first phase. The assumption of $\beta$ being a constant bounded away from $0$ captures the more interesting scenarios of bidding under RoS. For example, when the allocation and price functions of each query arise from a single-item second-price auction, the essence of the problem becomes how best to spend the extra slack $v_t-p_t(v_t)$ gained from queries with $v_t>p(v_t)$. If $\beta$ is tiny, or in the extreme case of $\beta=0$, there is nothing to optimize for, and the optimal solution would simply be $b_t=v_t$ all the time.

Since we need to only accrue a buffer of size $O(\sqrt{T}\log T)$, and the expected increment of the buffer is a constant $\beta$ per iteration, we can show the first phase finishes in $O(\sqrt{T}\log T)$ iterations in expectation. 
\begin{restatable}{proposition}{lemFirstPhaseLength}\label{lem:FirstPhaseLength}
Under \cref{assumption:beta} for the distribution $\calp$, let $K(\gseq)$ be the number of iterations in the first phase of \cref{alg:meta-strictconstraints-tcpa} for some input sequence $\gseq$. Then, we have
\[\E_{\gseq\sim \calp^T}[K(\gseq)]\leq O(\sqrt{T}\log T). \]
\end{restatable}

We also need the following technical statement on the difference in reward collected by \cref{alg:tcpa-truthful-soft} and $\opt$ for various lengths of input sequences. 
\begin{restatable}{proposition}{lemOptProptoLengths}\label{lem:OptProptoLengthOfSeq} Let $\gseq_{\!\ell} \sim \calp^{\ell}$ and $\gseq_{\!r} \sim \calp^{r}$ be  sequences of lengths $\ell$ and $r$, respectively, with $\ell\leq r$, of i.i.d. requests each from a distribution $\calp$. Then the following inequality holds. 
\[\E_{\gseq_{\!\ell} \sim \calp^{\ell}} \left[\rew(\cref{alg:tcpa-truthful-soft}, \gseq_{\!\ell})\right] \geq \frac{\ell}{r} \E_{\gseq_{\!r} \sim \calp^{r}} \left[ \rew(\opt,\gseq_{\!r}) \right] - O(\sqrt{r}).\]
\end{restatable}
The above two results help us bound the regret from the first phase, and we can use the guarantee from \cref{thm:MainThmTcpa} to bound the regret due to the second phase. Altogether we get the main result below.
\begin{theorem}\label{thm:strictTcpaOnly}
With i.i.d.\ inputs from a distribution $\calp$ over a time horizon $T$, the regret of \cref{alg:meta-strictconstraints-tcpa} on \cref{eq:tcpa-obj-simple}  is bounded by \[ \reg(\cref{alg:meta-strictconstraints-tcpa}, \calp^T) \leq O(\sqrt{T}\log T ).\] Further, there is no violation of the RoS constraint in \cref{eq:tcpa-obj-simple}. 
\end{theorem}
\begin{proof}
The claim on constraint violation follows by design of the algorithm: we collect a constraint violation buffer of at least $\viol$ before starting the second phase, in which we are guaranteed to violate the constraint by an additive factor of at most $\viol$.

We now prove the claimed regret bound by combining a lower bound on the expected reward and an upper bound on the expected optimum. To lower bound the algorithm's reward, we note that it is at least the reward from the second phase. 

Let $K(\gseq)$ be the random variable that represents the last iteration of the first phase, after which we run \cref{alg:tcpa-truthful-soft}. In this proof, we use $\gseq_{\!\!a:b}$ to denote the sequence $\gseq$ from time steps $a$ through $b$; when these end points do not matter, we simply denote a length $T$ sequence as $\gseq$. With this notation, we have for any sequence $\gseq$: \[   \rew(\cref{alg:meta-strictconstraints-tcpa}, \gseq) \geq \sum_{t = K(\gseq)+1}^T \vt \cdot \xt(\bt).\] Then taking expectations on both sides and using conditional expectations gives 
\begin{align*}
    \E_{\gseq\sim\calp^T} \left[\rew(\cref{alg:meta-strictconstraints-tcpa}, \gseq)\right] 
    &\geq \E_{k} \left[ \E_{\gseq\sim\calp^T \mid K(\gseq)=k} \left[ \sum_{t=k+1}^T \vt \cdot \xt(\bt) \mid K(\gseq) = k \right] \right]\\
    &= \E_{k} \left[ \E_{\gseq_{\!\!k+1:T}\sim \calp^{T-k}} \left[ \rew(\cref{alg:tcpa-truthful-soft}, \gseq_{\!\!k+1:T}) \right] \right]\\
    &\geq \E_k \left[\frac{T-k}{T}\cdot \E_{\gseq\sim \calp^T}\left[\rew(\opt,\gseq)\right] \right] - O(\sqrt{T}), \numberthis\label[ineq]{eq:AlgStrictRewLB1} 
\end{align*}
where the third step is due to the requests all being i.i.d. and the fact that we start running \cref{alg:tcpa-truthful-soft} fresh in the second phase, and the fourth step is by \cref{lem:OptProptoLengthOfSeq}. Finally, plugging \cref{eq:AlgStrictRewLB1} into the definition of $\reg$ from \cref{eq:defRegret} and simplifying gives 
\begin{align*}
     \reg(\cref{alg:meta-strictconstraints-tcpa}, \calp^T) &\leq \E_{\gseq\sim \calp^T} \left[K(\gseq)\right]\cdot \frac{\E_{\gseq\sim \calp^T}\left[\rew(\opt,\gseq)\right]}{T} + O(\sqrt{T}) \\
                    &\leq\E_{\gseq\sim \calp^T} \left[K(\gseq)\right] + O(\sqrt{T}) = O(\sqrt{T}\log T),
\end{align*} 
where in the third step we use $\rew(\opt,\gseq)\leq T$, and in the last step we use \cref{lem:FirstPhaseLength} and \cref{assumption:beta} that $\beta$ is a constant.
\end{proof} 
Although we structure the algorithm as the first phase followed by the second phase for the purpose of the analysis, the value of $\viol$ can be a pessimistic bound and make us run the first phase for an unnecessarily long time. A more practical implementation can break the first phase into smaller chunks and intermingle them with the execution of \cref{alg:tcpa-truthful-soft} on demand. That is, whenever we are about to violate the RoS constraint during \cref{alg:tcpa-truthful-soft}, we put it on hold and bid exactly the value $v_t$ for some iterations until we build up a certain amount of buffer and then resume the execution of \cref{alg:tcpa-truthful-soft}, where these special iterations are ignored (i.e., won't affect the dual variable updates). Intuitively this can perform better than \cref{alg:meta-strictconstraints-tcpa}, although without a rigorous regret guarantee.

\section{Bidding Under Both RoS and Budget Constraints}\label{sec:CombinedConstraints}
In this section, we combine our techniques from \cref{sec:softTCPAconstraint} and \cref{sec:hard_constraint} with those of \cite{BLM20} to obtain bidding algorithms that satisfy \emph{both} the RoS and budget constraints (i.e., \cref{combinedOptnProb}). More specifically, we propose \cref{alg:combined-truthful-soft} to solve \cref{combinedOptnProb}  under a strict imposition of the budget constraint and only an approximate one on the RoS constraint and \cref{alg:combined-truthful-strict} to strictly satisfy \emph{both} the RoS and budget constraints. 

\subsection{Approximate RoS and Strict Budget Constraints}
We start with a bidding algorithm that satisfies the budget constraint exactly and the RoS constraint up to some small violation in the worst case. Similar to the bidding rule in \cref{eq:bid-truthful-tcpa-soft}, the candidate bid for this algorithm is the one maximizing the price-adjusted reward, with one dual variable for each constraint:
\begin{align*}
    b_t=\arg\max_{b\geq 0} \left\{ \vt \cdot \xt(b) + \lambda_t\cdot (\vt \cdot \xt(b) - \pt(b)) - \mu_t\cdot \pt(b)\right\}
    &= \frac{1+\lambda_t}{\mu_t + \lambda_t} \cdot \vt,\numberthis\label{eq:candidateBidSoftRoSStrictBudget}
\end{align*} where the solution to the maximization holds by the definition of a truthful auction. Since we are in the setting where the budget constraint is strict, the candidate bid given by \cref{eq:candidateBidSoftRoSStrictBudget} is used as the final bid in this iteration only if we are not close to exhausting the total budget, as formalized in \cref{line:algCombined2updateBid} of \cref{alg:combined-truthful-soft}. Similar to \cref{alg:tcpa-truthful-soft}, the Lagrange multipliers $\lambda_t$ and $\mu_t$ serve to enforce the RoS and budget constraints respectively. 

We remark that \cref{alg:combined-truthful-soft} may be interpreted as combining our \cref{alg:tcpa-truthful-soft} with Algorithm 1 of \cite{BLM20}. The analysis of the regret bound also follows the outline of the main proof of \cite{BLM20}, integrating it with our regret bound for \cref{alg:tcpa-truthful-soft} from  \cref{thm:MainThmTcpa}. Intuitively the integration is straightforward since the analyses of both methods are linear in nature, allowing us to easily decompose the intermediate regret bound from the primal-dual framework into two components corresponding to the two constraints. We then bound them respectively with the same techniques from earlier sections to handle the RoS constraint and the techniques from \cite{BLM20} to address the budget constraint.

As for the constraint violation guarantees, the budget constraint is always satisfied by design, and the RoS violation bound follows from a simple corollary of \cref{lem:constraint_violation}, which applies here since  the extra budget constraint makes  our bid only \emph{more} conservative than that of \cref{alg:tcpa-truthful-soft}. We formally collect and state these results in \cref{lem:combinedRegret}, proving it in \cref{sec:combinedTCPAappendix}.

\begin{algorithm}[h!]
\caption{Bidding algorithm of the  agent operating under approximate RoS constraints and strict budget constraints  in a truthful auction (i.i.d.\ inputs)}\label[alg]{alg:combined-truthful-soft}

\begin{algorithmic}[1]
\State \textbf{Input:} Total time horizon $T$, requests $\gseq$ i.i.d.\ from the distribution $\calp^T$, total budget $\rho T$.
\State \textbf{Initialize:} Initial dual variable $\lambda_1=1$, $\mu_1=0$, total initial budget $B_1 := \rho T$,  dual mirror descent step size $\alpha = \frac{1}{\sqrt{T}}$ and $\eta = \frac{1}{(1+\rho^2)\sqrt{T}}$.

\For{$t=1,2,\cdots, T$}
    \State Observe the value $v_t$, and set the bid $\bt = \twopartdef{\frac{1+  \lambda_{t}}{\mu_t + \lambda_{t}}\cdot v_t}{B_t\geq 1}{0}{\text{otherwise}}$. \label{line:algCombined2updateBid}

    \State Compute  $g_t(b_t) = v_t \cdot x_t({b}_t) - p_t(b_t)$.
 \label{line:algCombined2g}

    \State Update the dual variable of the RoS constraint as $\lambda_{t+1} = \lambda_t \exp\left[-\alpha \cdot\gt(\bt)\right]$.  \label{line:algCombined2lambda}
    
    \State Compute $g_t^{\prime}(\bt) = \rho - \pt(\bt)$. 
   \label{line:algCombinedFixedBudgetDualGrad}

    \State Update the dual variable of the budget constraint as $    \mu_{t+1} := \mathrm{Proj}_{\mu \geq 0}(\mu_t - \eta \cdot g_t^{\prime}(\bt))$.  \label{line:algCombinedFixedBudgetDualMirror}
    
    \State Update the leftover budget $B_{t+1} = B_t - \pt(\bt)$. 
\EndFor
\State \Return The sequence $\{b_t\}_{t=1}^T$ of bids. 
\end{algorithmic} 
\end{algorithm}

\begin{restatable}{theorem}{thmCombinedRegretSoftTcpaStrictBudget}\label{lem:combinedRegret}
With i.i.d.\ inputs from a distribution $\calp$ over a time horizon $T$, the regret of \cref{alg:combined-truthful-soft}  on \cref{combinedOptnProb} is bounded by \[\reg(\cref{alg:combined-truthful-soft}, \mathcal{P}^T) \leq O(\sqrt{T}).\] Further, \cref{alg:combined-truthful-soft} incurs a violation of at most $O(\sqrt{T}\log T)$ of the RoS constraint and no violation of the budget constraint. 
\end{restatable}

\subsection{Strict RoS and Strict Budget Constraints}\label{sec:strictTCPAstrictBudget}
In the case when we impose both strict RoS and strict budget constraints, we essentially combine the key ideas from \cref{alg:meta-strictconstraints-tcpa} and \cref{alg:combined-truthful-soft}: We keep bidding the value until we accumulate a sufficient buffer on the RoS constraint; following this phase, we run \cref{alg:combined-truthful-soft}, which, as explained in the preceding section, imposes strict budget and approximate RoS constraints. 

By the same reasoning as for \cref{alg:meta-strictconstraints-tcpa}, the RoS constraint is not violated; the budget constraint is also respected via \cref{line:exitCase} of \cref{alg:combined-truthful-strict}. The regret analysis follows a strategy similar to that of the proofs of \cref{lem:combinedRegret} and \cref{thm:strictTcpaOnly}. Our main result of this section  is in \cref{lem:combinedRegretBothStrict}, with its proof in \cref{sec:strictCombinedAppendix}.   

\begin{algorithm}[h!]
\caption{Bidding algorithm of the agent operating under strict RoS and strict budget constraints in truthful auctions (i.i.d.\ inputs)}\label{alg:combined-truthful-strict}
\begin{algorithmic}[1]
\State \textbf{Input:} Total time horizon $T$, requests $\gseq$ i.i.d.\ from the distribution $\calp^T$, total budget $\rho T$. 

\State \textbf{Initialize:} Set the initial buffer $g_0(b_0) = 0$, $\viol = 2\sqrt{T}\log T$, initial total budget $B_1 = \rho T$, and iteration $t = 1$.

\While{$\sum_{i=0}^{t-1} g_i (b_i) \leq \viol$} \Comment{First phase} 
    \State Observe the value $v_t$, and set the bid $b_t = v_t$.\label{line:MetaStrictAlgTCPABid}
    \State Observe the price $p_t$ and the allocation $x_t$ at $\bt$, and compute $g_t(\bt) := \vt \cdot \xt(\bt) - \pt(\bt)$. 
    \State Update the total budget to $B_{t+1} = B_t - p_t(b_t)$.
    \State Increment the iteration count $t= t+1$.
    \If {$t\geq \rho T$} \label{line:exitCase}
        \State Exit algorithm. 
    \EndIf
\EndWhile

\State Run \cref{alg:combined-truthful-soft}  with time horizon $T-t$ and the remaining $T-t$ requests from $\gseq$ as input and initial total budget $B_t$. \Comment{Second phase}

\State \Return The sequence $\{\bt\}_{t=1}^T$ of generated bids. 

\end{algorithmic}
\end{algorithm}

\begin{restatable}{theorem}{thmCombinedRegretStrictTcpaStrictBudget}\label{lem:combinedRegretBothStrict}
With i.i.d.\ inputs from a distribution $\calp$ over a time horizon $T$, the regret of \cref{alg:combined-truthful-strict}  on \cref{combinedOptnProb} is bounded by \[\reg(\cref{alg:combined-truthful-strict}, \mathcal{P}^T) \leq O(\sqrt{T}\log T).\] Further, \cref{alg:combined-truthful-strict}  suffers no constraint violation of either the RoS or budget constraint. 
\end{restatable}

\section*{Acknowledgement}
We thank Aranyak Mehta for suggesting and help formulating the problem. We are grateful to  Santiago Balseiro, Kshipra Bhawalkar Lane, Haihao Lu, Vahab Mirrokni, and Balasubramanian Sivan for helpful discussions throughout the project. 

\newpage

\printbibliography
\newpage

\begin{appendices} 
In this section, we provide all the proofs that have been omitted from the main body. The proofs appear here in the order in which the corresponding results are stated in the main body and are organized by sections corresponding to the main body. 

\section{Proofs of the Approximate RoS Constraint}\label[app]{sec:pureTCPAappendix}

\begin{proposition}\label{prop:gradientProperties}
Let $g_t$ be as defined in \cref{eq:defGradTcpa} and $\pt$ be as defined in \cref{eq:MyersonsLemma}. Let $\bt\leq \frac{1+\lambda_t}{\lambda_t} \cdot \vt$. Then we have \[\max(-1/\lambda_t, -1) \leq g_t(b_t) \leq v_t \cdot x_t (b_t).\] 
\end{proposition}
\begin{proof} The non-negativity of $v_t$ and $x_t$ along with the bound $p_t(b)\leq 1$ immediately give $g_t(b_t) \geq -1$. Further, the non-negativity of $p_t$ from \cref{eq:MyersonsLemma} gives the upper bound  $g_t(\bt)\leq \vt \cdot \xt(\bt)$. Using the exact expression for $p_t$ from \cref{eq:MyersonsLemma}, along with non-negativity of $x_t$, the stated assumption on $b_t$, and $v_t, x_t\leq 1$ gives
gives 
\begin{align*} 
    g_t(b_t) &= (v_t- b_t)\cdot x_t(b_t) + \int_{z = 0}^b x_t(z) dz \geq (v_t - b_t) \cdot x_t(b_t) \geq -\frac{v_t}{\lambda_t} \cdot x_t(b_t) \geq - \frac{1}{\lambda_t}. 
\end{align*} 
\end{proof}

\lemWeakDuality*

The proof of this result essentially follows the outline in \cite{BLM20}. We prove it in two parts: an upper bound on the expected optimal value followed by a lower bound on the expected reward of \cref{alg:tcpa-truthful-soft}. To prove the former, we go through the primal-dual framework, for which need the following technical definitions. 
\begin{definition}\label[defn]{defn:Defs-tCPA-only} We define the following dual variables. 
\begin{itemize}
\item Let $f_t(b):= \vt \cdot \xt(b)$, recall $g_t$ from \cref{eq:defGradTcpa}, for some fixed $\lambda\geq 0$, define
\begin{equation}
\fts(\lambda):=\max_{b\geq 0}\left[f_t(b)+\lambda\cdot \gt(b)\right].\label{eq:1-tcpaOnly}
\end{equation}
\item  Let $\fs$ be defined as in \cref{eq:1-tcpaOnly}. Then we define the following dual variable parametrized by the input distribution $\calp$. 
\begin{equation}
\bard(\lambda|\calp):=\E_{(v,p)\sim\calp}\left[\fs_{\mathsf{RoS}}(\lambda)\right].\label{eq:2-tcpaOnly}
\end{equation} 
\end{itemize} 
\end{definition}

The proof outline is to first show an upper bound on the expected optimal reward via weak duality then a lower bound on the expected reward collected by our algorithm. 
\begin{proposition}\label{lem:weakDualityUpperBoundOPT}
Recall  $\bard(\lambda|\calp)$ as defined in \cref{eq:2-tcpaOnly}. Then the optimum value $\rew(\opt,\gseq)$ for \cref{eq:tcpa-obj-simple} defined for a sequence $\gseq_{\!\ell}\sim\calp^{\ell}$ of $\ell$ requests 
satisfies the inequality \[\E_{\gseq_{\!\ell}\sim\calp^{\ell}}\left[\rew(\opt,\gseq_{\!\ell})\right]\leq \ell\cdot\min_{\lambda\geq 0}\bard(\lambda|\calp).\]
\end{proposition}
\begin{proof}
We have, by the definition of $\rew(\opt,\gseq)$ in \cref{eq:defOPTtcpa} and the definition of $g_t$ in \cref{eq:defGradTcpa} that 
\begin{align*}
    \rew(\opt,\gseq_{\!\ell}) &= \max_{\{b_t\}: \sum_{t=1}^{\ell} g_t(b_t) \geq 0} \sum_{t=1}^{\ell} f_t(b_t)\\
                &= \max_{\{b_1, b_2, \dots, b_{\ell}\}} \left\{ \sum_{t = 1}^{\ell} f_t(b_t) + \min_{\lambda\geq 0} \lambda \cdot \left[\sum_{t=1}^{\ell} g_t(b_t)\right] \right\}\\
                &= \max_{\{b_1, b_2, \dots, b_{\ell}\}} \min_{\lambda \geq 0} \left\{\sum_{t = 1}^{\ell} \left[ f_t(b_t) + \lambda \cdot g_t(b_t) \right]\right\}. 
\end{align*}
By Sion's min-max theorem and by the definition of $\fts$ in \cref{eq:1-tcpaOnly}, we have that 
\begin{align*} 
               \rew(\opt,\gseq_{\!\ell}) &\leq \min_{\lambda \geq 0} \max_{\{b_1, b_2, \dots, b_{\ell}\}} \left\{ \sum_{t = 1}^{\ell} \left[ f_t(b_t) + \lambda  \cdot g_t(b_t)\right]\right\} \\
                &\leq \min_{\lambda \geq 0} \left\{ \sum_{t = 1}^{\ell} \max_{b_t} \left[ f_t (b_t) + \lambda  \cdot g_t(b_t) \right]\right\} \\
                &=\min_{\lambda \geq 0} \left\{\sum_{t=1}^{\ell} \fts(\lambda)\right\} \leq \sum_{t = 1}^{\ell} \fts (\lambda^{\prime}), \text{ for any $\lambda^{\prime} \geq 0$}. \numberthis\label[ineq]{eq:weakDuality}
\end{align*} 
Now taking expectations\footnote{The expectation is over the randomness in the sequence of requests $\gseq_{\!\ell}$ over the time horizon $\ell$.} on both sides of \cref{eq:weakDuality} and using the definition of $\bard$ from \cref{eq:2-tcpaOnly} gives, for any fixed $\lambda^{\prime} \geq 0$, 
\begin{align*}
\mathbb{E}_{\gseq_{\!\ell} \sim \mathcal{P}^{\ell}} \left[\rew(\opt,\gseq_{\!\ell})\right] \leq \mathbb{E}_{\gseq_{\!\ell} \sim \mathcal{P}^{\ell}} \left[\sum_{t = 1}^{\ell} \fts (\lambda^{\prime})\right] = \sum_{t = 1}^{\ell} \mathbb{E}_{\gseq_{\!\ell} \sim \mathcal{P}^{\ell}} \left[ \fts(\lambda^{\prime})\right] &= \ell \cdot  \bard(\lambda^{\prime}| \mathcal{P}).
\end{align*} where the final equation crucially uses the fact that all $t$ requests are drawn i.i.d.\ from the same distribution $\mathcal{P}$ and the argument $\lambda^{\prime}>0$ is fixed. Therefore, in particular, the preceding inequality holds for the specific $\lambda^{\prime}$ that minimizes the right-hand side, thus finishing the proof. 
\end{proof}

\begin{proposition}\label{lem:weakDualityLowerBoundALG}
For some fixed number $r$, let $\overline{\lambda}_{r} := \frac{1}{r} \sum_{t = 1}^r \lambda_t$, where $\lambda_t$ are the dual iterates in \cref{alg:tcpa-truthful-soft}. Then,  the  reward (see \cref{eq:defReward}) of \cref{alg:tcpa-truthful-soft} is lower bounded as \begin{align*}
\rew(\cref{alg:tcpa-truthful-soft},\gseq_{\!r}) &\geq\E_{\gseq_{\!r}\sim\calp^{r}}\left[r\cdot\bard(\overline{\lambda}_r|\calp)\right] -\E_{\gseq_{\!r}\sim\calp^r}\left[\sum_{t=1}^{r}\lamt\cdot g_t(b_t)\right].
\end{align*}
\end{proposition}
\begin{proof}
Recall we have by \cref{line:tcpaAlgBid} in \cref{alg:tcpa-truthful-soft} and by the definition of $\fts$ in \cref{eq:1-tcpaOnly},  \[ f_t(b_t) = \fts(\lambda_{t})  - \lambda_{t} \cdot g_t(b_t).\] Taking expectations on both sides by conditioning on the randomness until (and including) iteration $t-1$, we have, 
\begin{align*} 
    \mathbb{E} \left[ f_t(b_t) | \sigma_{t-1}\right] &=  \mathbb{E}\left[ \fts(\lambda_t)|\sigma_{t-1}\right] - \mathbb{E}\left[ \lambda_t \cdot g_t(b_t) | \sigma_{t-1}\right]\\ 
                        &=  \mathbb{E}_{(v, p) \sim \mathcal{P}} \left[ f^{\star}_{\mathsf{RoS}}(\lambda_t)\right] - \mathbb{E}\left[ \lambda_t \cdot g_t(b_t) | \sigma_{t-1} \right]\\                            &= \bard(\lambda_t|\mathcal{P}) - \mathbb{E}\left[ \lambda_t \cdot  g_t(b_t)|\sigma_{t-1}\right],
\end{align*} where the second step used the fact that in \cref{alg:tcpa-truthful-soft}, the $\lambda_t$ is fixed when $\sigma_{t-1}$ is, and the final step used the definition of $\bard$ in \cref{eq:2-tcpaOnly}.
Summing over $t= 1, 2, \dotsc, r$ and taking expectations (this is a valid operation since $r$ is a fixed number) gives, 
\begin{align*}
    \mathbb{E}_{\gseq_{\!r}\sim \calp^r}\left[ \sum_{t = 1}^r f_t(b_t) \right] &= \mathbb{E}_{\gseq_{\!r} \sim \calp^r}\left[ \sum_{t = 1}^r \bard(\lambda_t|\mathcal{P})\right] - \mathbb{E}_{\gseq_{\!r} \sim \calp^r}\left[\sum_{t = 1}^r \lambda_t  \cdot g_t(b_t) \right],
\end{align*} and we may finish the proof by invoking convexity of $f^{\star}_{\mathsf{RoS}}$. 
\end{proof}

\begin{proof}[Proof of \cref{lem:regretBoundLambdatGt}]
We begin by restating the result from \cref{lem:weakDualityUpperBoundOPT} for an input sequence of length $T$: \[ \E_{\gseq\sim\calp^{T}}\left[\rew(\opt,\gseq)\right]\leq T\cdot\min_{\lambda\geq 0}\bard(\lambda|\calp). \numberthis\label[ineq]{eq:temp1}\] The minimum on the right-hand side of the preceding inequality can be further bounded as follows for any $\lambda(\gseq) \geq 0$. 
\[
T \cdot \min_{\lambda \geq 0} \bard(\lambda|\mathcal{P}) \leq  
T \cdot \mathbb{E}_{\gseq \sim \mathcal{P}^T}\left[\bard(\lambda(\vec \gamma)|\mathcal{P})\right]. \numberthis\label[ineq]{eq:boundMinByExp}
\]
In particular, then, we may choose $\lambda(\gseq) = \frac{1}{T}\sum_{t = 1}^T \lambda_t := \overline{\lambda}_T$ on the right-hand side of \cref{eq:boundMinByExp} and combine with \cref{eq:temp1} to obtain   
\begin{align*} 
    \mathbb{E}_{\gseq \sim \mathcal{P}} \left[\rew(\opt,\gseq)\right] \leq \mathbb{E}_{\gseq\sim \mathcal{P}^T} \left[ T\cdot \bard(\overline{\lambda}_T | \mathcal{P}) \right]. \numberthis\label[ineq]{eq:SubtleIneq}
\end{align*} Combining this with \cref{lem:weakDualityLowerBoundALG} applied to an input sequence of length $T$ finishes the proof. 
\end{proof}

\thmMainTCPA* 
\begin{proof}
Plugging \cref{lem:regret_bound} into  \cref{lem:regretBoundLambdatGt} gives the regret bound. To see the claim on constraint violation, we first note from \cref{prop:gradientProperties} that the gradients $g_t$ in \cref{alg:tcpa-truthful-soft} satisfy $g_t(b_t) \geq -1/\lambda_t$. Therefore, the result of \cref{lem:constraint_violation} applies; note that constraint violation of \cref{alg:tcpa-truthful-soft} corresponds precisely to $-\sum_{t=1}^T g_t(\bt)$. 
\end{proof}

\subsection{Approximate RoS Constraints Using the Squared Mirror Map}\label[app]{sec:squaredMirrorMap}

\begin{algorithm}[h!]
\caption{Algorithm for bidding under a RoS constraint in truthful auctions with squared regularizer}\label{alg:tcpa-truthful-soft-squared}
\begin{algorithmic}[1]
\State \textbf{Input:} Total time horizon $T$ and requests $\gseq$ from the data distribution $\calp^T$
\State \textbf{Initialize:} Dual variable $\lambda_1=1$, step size for dual update $\alpha$, and mirror map $h({}\cdot{}) = \frac{1}{2}\|{}\cdot{}\|^2_2$. 

\For{$t=1,2,\cdots, T$}
    \State Observe the realization of the valuation $v_t$, and set the bid: 
        $b_t =  \frac{v_t}{\lambda_t} + v_t$.

    \State Observe $\pt$ and $\xt$ evaluated at the chosen bid $\bt$ and compute $ g_t(\bt) = v_t \cdot x_t({b}_t) - p_t({b}_t).$

    \State Update the dual variable corresponding to the RoS constraint: $\lambda_{t+1} = (\lambda_t - \alpha g_t(b_t))_+$
\EndFor

\State \Return The sequence $\{\bt\}_{t=1}^T$ of generated bids. 
\end{algorithmic}
\end{algorithm}
\begin{lem}\label[lem]{lem:expectedRegretBoundWithSquaredReg}
The regret, as defined in \cref{eq:defRegret}, of \cref{alg:tcpa-truthful-soft-squared}, run on \cref{combinedOptnProb}, with i.i.d.\ inputs from distribution $\calp$ over a time horizon $T$ is bounded from above by  $2\alpha T+\alpha^{-1}$. 
\end{lem} 
\begin{proof} We observe that the regularizer $h(u) = \frac{1}{2}u^2$ is  $1$-strongly convex on the non-negative orthant, which implies we can use \cref{lem:regretBoundLambdatGt} and mirror descent error bounds (\cref{lem:onlineMD}) to bound the average regret as \[\reg(\cref{alg:tcpa-truthful-soft-squared}, \mathcal{P}^T) \leq \mathbb{E}_{\gseq \sim \mathcal{P}^T} \left[ \sum_{t = 1}^T \lambda^* \cdot g_t(b_t)\right] +  2 \alpha T + \frac{V_h(\lambda^{\ast}, \lambda_1)}{\alpha}.\] By choosing
$\lambda^{\ast} = 0$ (and using $\lambda_1 = 1$) in this bound yields the claimed regret bound. 
\end{proof}

\begin{lem}\label[lem]{lem:BoundOnConstraintViolation}
Running \cref{alg:tcpa-truthful-soft-squared}  on \cref{eq:tcpa-obj-simple}  yields a constraint violation of at most $\frac{T^c}{\alpha} + T^{1-c}$ for some parameter $c>0$. 
\end{lem}
\begin{proof} Based on the dual update 
rule  in \cref{alg:tcpa-truthful-soft-squared}, $\lambda_{t+1} \geq \lambda_t - \alpha g_t(b_t)$, as  a result of which $g_t(b_t) \geq \frac{1}{\alpha}(\lambda_t - \lambda_{t+1})$. Let $ c> 0$. Let $T^\prime$ be the last time step at which $\lambda_{T^\prime} \leq T^c$. Then,  $\sum_{t=1}^{T^{\prime}} g_t(b_t) \geq -{T}^c/\alpha$  and applying $g_t(b_t) \geq -1/\lambda_t$ (which comes from \cref{prop:gradientProperties}) with $\lambda_t \geq T^c$ for $t\geq T^{\prime}$ gives $\sum_{t=T^{\prime}}^T g_t(b_t) \geq -T^{1-c}$. Therefore, we have $\sum_{t=1}^T g_t(b_t) \geq -T^c/\alpha - T^{1-c}$. 
\end{proof}

\begin{corollary}
The regret, defined in \cref{eq:defRegret}, of \cref{alg:tcpa-truthful-soft-squared} run on \cref{eq:tcpa-obj-simple}, with i.i.d.\. inputs from a distribution $\calp$ over a time horizon $T$ is bounded from above by
\[
\reg(\cref{alg:tcpa-truthful-soft-squared}, \mathcal{P}^T)\leq   O(T^{2/3}). \numberthis\label[ineq]{eq:MainRegretBoundtCPSquared}
\] Further, the violation of the RoS constraint in \cref{eq:tcpa-obj-simple} is at most $O(T^{2/3})$. 
\end{corollary} 
\begin{proof}
By picking $\alpha= T^{-1/3}$ in \cref{lem:expectedRegretBoundWithSquaredReg}  yields a regret bound of $O(T^{2/3})$ and picking $c = 1/3 $ in \cref{lem:BoundOnConstraintViolation} along with the same $\alpha$ gives a maximum constraint violation of $O(T^{2/3})$. 
\end{proof} 

\section{Proofs of the Strict RoS Constraint}\label[app]{sec:strictTCPA}
\lemOptProptoLengths*
\begin{proof}We have by  \cref{lem:weakDualityLowerBoundALG}, \cref{lem:weakDualityUpperBoundOPT}, and  \cref{lem:regret_bound}: 
\begin{align*}
    \rew(\cref{alg:tcpa-truthful-soft}, \gseq_{\!\ell}) &\geq\E_{\gseq_{\!\ell}\sim \calp^{\ell}} \left[\ell\cdot\bard(\overline{\lambda}_{\ell}|\calp)\right] -\E_{\gseq_{\!\ell}\sim \calp^{\ell}}\left[\sum_{t=1}^{\ell}\lamt\cdot g_t(b_t)\right]\\
     &\geq \ell \cdot \min_{\lambda\geq0} \bard(\lambda|\calp) -\E_{\gseq_{\!\ell}\sim \calp^{\ell}}\left[\sum_{t=1}^{\ell}\lamt\cdot g_t(b_t)\right]\\
     &\geq \frac{\ell}{r} \cdot \E_{\gseq_{\!r} \sim \calp^r}\left[ \rew(\opt, \gseq_{\!r}) \right] - \E_{\gseq_{\!\ell}\sim \calp^{\ell}}\left[\sum_{t=1}^{\ell}\lamt\cdot g_t(b_t)\right], \\
     &\geq \frac{\ell}{r} \cdot \E_{\gseq_{\!r} \sim \calp^r}\left[ \rew(\opt, \gseq_{\!r}) \right]  - O(\sqrt{r}). 
\end{align*} 
\end{proof}

\lemFirstPhaseLength*
\begin{proof}
Let $z_t = \max(0, v_t \cdot x_t(b_t) - p_t(b_t))$ be the reward collected at iteration $t$ (in the first phsae). Let $z_t^{\prime}  := \frac{\beta}{2}\mathbf{1}_{z_t \geq \beta/2}$. By construction, $z_t^{\prime} \leq z_t$ for all $t$. For any given sequence $\gseq$, let $K(\gseq)$ and $K^\prime(\gseq)$, respectively, be the first time such that $\sum_{t=1}^{K(\gseq)} z_t \geq R$ and $\sum_{t=1}^{K^\prime(\gseq)} z_t^\prime \geq R$ for some reward $R$. Then, for every $\gseq$, we have $K(\gseq) \leq K^\prime(\gseq)$. 

By the boundedness assumption on $z_t$, we have \[\Prob(z_t^\prime = \beta/2) = \Prob(z_t \geq \beta/2) \geq \E \left[z_t\right] - \Prob(z_t \leq \beta/2)\cdot \E \left[z_t | z_t \leq \beta/2\right]  \geq \beta/2.\]

Then, by Hoeffding bound, 
\[ \Prob(K(\gseq)\geq q) \leq\Prob(K^\prime(\gseq)\geq q) \leq e^{-O(q\beta^2)}. \numberthis\label[ineq]{eq:hpBoundOnK}\]

Picking $q = O(R/\beta^2)$ where $R = 2\sqrt{T}\log T$ gives the claimed bound on $\E\left[K(\gseq)\right]$.  
\end{proof}

\section{Proofs for Both RoS and (Strict) Budget Constraints}

\subsection{Lemmas for Approximate RoS and Strict Budget Constraints}\label[app]{sec:combinedTCPAappendix}
The following definition generalizes \cref{defn:Defs-tCPA-only}. 
\begin{definition}We need the following technical definitions of dual variables. 
\begin{itemize}
\item For some $\lambda\geq 0$ and $\mu \geq 0$,  let $f_t(b):= \vt \cdot \xt(b)$, define $g_t$ as in \cref{eq:defGradTcpa}, and define
\begin{equation}
\ftscombined(\mu, \lambda):=\max_{b}\left[f_t(b)+\lambda\cdot \gt(b) )-\mu\cdot \pt(b)\right].\label{eq:1-combined}
\end{equation}
\item  The following dual variable parametrized by $\rho$ and $\calp$; the quantity $\fs$ is defined in the same way as in \cref{eq:1-combined}.  
\begin{equation}
\bardcombined(\mu, \lambda|\calp, \rho):=\mu \cdot \rho + \E_{(v,p)\sim\calp}\left[\fs_{\mathsf{combined}}(\mu,\lambda)\right].\label{eq:2-combined}
\end{equation} 
\end{itemize}
\end{definition}

\begin{restatable}{proposition}{lemWeakDualityCombined}\label{lem:weakdualityCombinedCase}
For some $\rho^{\prime} \geq 0$, let $\bardcombined(\mu,\lambda|\calp, \rho^{\prime})$ be as defined in \cref{eq:2-combined}. Then the optimum value $\rew(\opt,\gseq_{\!\ell}, \rho)$ for \cref{combinedOptnProb} with a total initial budget of $\rho \ell$ over a sequence $\gseq_{\!\ell}\sim\calp^{\ell}$ of $\ell$ requests 
satisfies the inequality \[\E_{\gseq_{\!\ell}\sim\calp^{\ell}}\left[\rew(\opt,\gseq_{\!\ell}, \rho)\right]\leq \ell\cdot\min_{\mu\geq0,\lambda\geq 0}\left[\bardcombined(\mu,\lambda|\calp, \rho^{\prime}) + (\rho-\rho^{\prime})\cdot \mu\right].\]
\end{restatable}

\begin{proof}
In this proof, we essentially repeat the ideas in the proof of \cref{lem:weakDualityUpperBoundOPT}. First, by definition, we have 
\begin{align*}
 \rew(\opt, \gseq_{\!\ell}, \rho)&:=\max_{\left\{ \bt\right\} :\sum_{t=1}^{\ell}g_t(\bt)\geq 0,\sum_{t=1}^{\ell}\pt(\bt)\leq\rho \ell}\left[\sum_{t=1}^{\ell}\vt\cdot\xt(\bt)\right]\\
 & =\max_{\left\{ \bt\right\} }\min_{\lambda\geq0,\mu\geq0}\left[\sum_{t=1}^{\ell}\vt\cdot\xt(\bt)+\lambda\cdot\sum_{t=1}^{\ell}\gt(\bt)+\mu\cdot(\ell\rho-\sum_{t=1}^{\ell}\pt(\bt))\right] 
 \end{align*}
 By application of Sion's minimax theorem and the definition of $\ftscombined$, we get 
 \begin{align}  
\rew(\opt, \gseq_{\!\ell}, \rho) 
& \leq\min_{\lambda\geq0,\mu\geq0}\sum_{t=1}^{\ell}\mu\cdot\rho+\max_{b_{t}}\left[\vt\cdot\xt(\bt)+\lambda\cdot\gt(\bt)-\mu\cdot\pt(\bt)\right]\nonumber \\
 & =\min_{\lambda\geq0,\mu\geq0}\sum_{t=1}^{\ell}\left[\rho\cdot\mu+\ftscombined(\mu,\lambda)\right]\leq\sum_{t=1}^{\ell}\left[\rho\cdot\mup+\ftscombined(\mup,\lamp)\right],\label[ineq]{eq:1}
\end{align}
for some fixed $\mup\geq0$ and $\lamp\geq0$. Then, taking the expectations
on both sides of \cref{eq:1} and using the fact that $\vt,\xt,\pt$
are all drawn from i.i.d.\ distributions, we get
\begin{align*}
\E_{\gseq_{\!\ell}\sim\calp^{\ell}}[\rew(\opt,\gseq_{\!\ell}, \rho)]  & \leq\sum_{t=1}^{\ell}\left[\rho\cdot\mup+\E_{(v,p)\sim\calp}\left[\fs_{\mathsf{combined}}(\mup,\lamp)\right]\right]\\
 & =\ell\cdot\bardcombined(\mup,\lamp|\calp, \rho^{\prime}) + \ell(\rho - \rho^{\prime})\cdot \mu^{\prime}.
\end{align*} Therefore, in particular, the preceding inequality holds for the $\lambda, \mu\geq 0$ that minimize the bound, thus finishing the proof. 
\end{proof}

Next, to  analyze the reward collected by \cref{alg:combined-truthful-soft}, similar to \cite{BLM20}, which gives an algorithm for a strict budget constraint, we need the following notion of stopping time. 

\begin{definition}\label[defn]{defn:Defs}
The stopping time $\tau$ of \cref{alg:combined-truthful-soft}, with a total initial budget of $B$ is the first time $\tau$ at which $$\sum_{t = 1}^{\tau} \pt(\bt) + 1 \geq B.$$ Intuitively, this is the first time step at which the total price paid almost exceeds the total budget. 
\end{definition}
To prove our main regret bound in \cref{lem:combinedRegret}, we first prove \cref{lem:minimumRegretCombinedCase}, which gives the minimum expected reward of \cref{alg:combined-truthful-soft}. The proof follows along the lines of that in \cite{BLM20} and \cref{lem:weakDualityLowerBoundALG}. 

\begin{restatable}{proposition}{lemAverageDualVarLowerBound}\label{lem:minimumRegretCombinedCase}
Let  $\tau$ be a stopping time as defined in \cref{defn:Defs} for some initial budget $\rho^{\prime} k$. Then the expected reward (see \cref{eq:defReward}) of \cref{alg:combined-truthful-soft} over a sequence of length $k$ with i.i.d.\ input requests from distribution $\calp^k$ is lower bounded as \begin{align*}
\E_{\gseq_{\!k}\sim \calp^k}\left[\rew(\cref{alg:combined-truthful-soft},\gseq, \rho^{\prime})\right] &\geq\E_{\gseq_{\!k}\sim\calp^{k}} \left[\tau\cdot\bardcombined(\mub,\lamb|\calp, \rho^{\prime})\right]\\
&\quad -\E_{\gseq_{\!k} \sim \calp^k} \left[\sum_{t=1}^{\tau}\mu_{t}\cdot(\rho^{\prime}-\pt(\bt))- \sum_{t=1}^{\tau}\lamt\cdot g_t(b_t)\right], 
\end{align*} where $\mub = \frac{1}{\tau}\sum_{i=1}^{\tau} \mu_i$ and $\lamb = \frac{1}{\tau} \sum_{i=1}^{\tau} \lambda_i$. 
\end{restatable}
\begin{proof}
By \cref{line:algCombined2updateBid} in \cref{alg:combined-truthful-soft}, we have, until the stopping time $t = \tau$, 
\[
\ftscombined(\mut,\lamt):=\vt\cdot\xt(\bt) + \lamt\cdot\gt(\bt) - \mut\cdot\pt(\bt).
\]
Rearranging the terms and taking expectations conditioned on the randomness up to step $t-1$,
\begin{align}
\E_{\gseq_{\!k}\sim\calp^{k}}\left[\vt\cdot\xt(\bt)|\sigma_{t-1}\right] & =\E_{\gseq_{\!k}\sim\calp^{k}}\left[\ftscombined(\mut,\lamt) + \mu_{t}\cdot\pt(\bt) - \lamt\cdot\gt(\bt)|\stm\right]\label{eq:FirstBoundForCombinedRewardLB}.  
\end{align} 
Per \cref{line:algCombinedFixedBudgetDualGrad} in \cref{alg:combined-truthful-soft} that once we fix the randomness up to $t-1$, the dual variables are all fixed, which gives us 
\begin{align*} 
\E_{\gseq_{\!k}\sim\calp^{k}}\left[\ftscombined(\mut,\lamt) | \stm\right] &=\E_{(v,p)\sim\calp}\left[\fs_{\mathsf{combined}}(\mut,\lamt)\right] \numberthis\label{eq:ftsCombinedTofsCombinedIID}
\end{align*}
Combining \cref{eq:ftsCombinedTofsCombinedIID}  with the definition of $\bardcombined$ in \cref{eq:2-combined} and plugging back into \cref{eq:FirstBoundForCombinedRewardLB} then gives 
\begin{align*} 
\E_{\gseq_{\!k}\sim\calp^{T}}\left[\vt\cdot\xt(\bt)|\sigma_{t-1}\right] 
 & =\bardcombined(\mut,\lamt|\calp, \rho^{\prime})-\E_{\gseq_{\!k}\sim\calp^{k}}\left[\mu_{t}\cdot(\rho^{\prime}-\pt(\bt)) + \lamt\cdot\gt(\bt)|\stm\right]
\end{align*}
Summing over $t=1,2,\dots,\tau,$ and using the Optional Stopping Theorem, we get 
\begin{align*}
\E_{\gseq_{\!k}\sim\calp^{k}}\left[\sum_{t=1}^{\tau}\vt\cdot\xt(\bt)\right] & =\E_{\gseq_{\!k}\sim\calp^{k}}\left[\sum_{t=1}^{\tau} \left(\bardcombined(\mut,\lamt|\calp, \rho^{\prime}) - \mu_{t}\cdot(\rho^{\prime}-\pt(\bt)) - \lamt\cdot\gt(\bt) \right)\right].
\end{align*} 
We finish the proof by using the convexity of $\bardcombined$ 
in the preceding equation.
\end{proof}

Our bound on regret requires the following technical result bounding one of the terms arising in \cref{lem:minimumRegretCombinedCase}. 
\begin{proposition}\label{prop:BoundingRhoTermMirrorDescent}
Consider a run of \cref{alg:combined-truthful-soft} with initial total budget $\rho \ell$ and the total time horizon $\ell$. We define the corresponding stopping time (as defined in \cref{defn:Defs}) as the time $\tau$ at which $\sum_{t=1}^{\tau} p_t(b_t) \geq \rho \ell-1$. Then, the dual variable $\{\mu_t\}$ that evolves as per \cref{line:algCombinedFixedBudgetDualMirror} in \cref{alg:combined-truthful-soft} satisfies the following inequality. \[\sum_{t=1}^{\tau} \mu_t \cdot (\rho - p_t(b_t)) \leq (\tau - \ell) + 1/\rho + O(\sqrt{\ell}). \] 
\end{proposition}
\begin{proof}
To bound $\sum_{t=1}^{\tau}\mu_{t}\cdot(\rho-\pt(\bt)),$
we observe that  the mirror descent guarantee of \cref{lem:onlineMD} applies to give, for  any $\mu \geq 0$, 
\[
\sum_{t=1}^{\tau}\mu_{t}\cdot(\rho-\pt(\bt))\leq\sum_{t=1}^{\tau}\mu\cdot(\rho-\pt(\bt))+\err(\tau,\eta), \numberthis\label[ineq]{eq:FixedBudgetDualVarMDGuarantee}\]
where  $\err(R, \eta) := \frac{1}{\sigma}(1+{\rho}^2) \eta R + \frac{1}{\eta} V_h(\mu, \mu_1)$, where $h(u) = \frac{1}{2} u^2$, $\sigma = 1$, and  $\mu_1 = 0$. To finish the proof, we choose
$\mu=1/\rho$, use $\sum_{t=1}^{\tau} p_t(b_t) \geq \rho\ell -1$, and choose $\eta = \frac{1}{(1+\rho^2)\sqrt{\ell}}$ in $\err(\tau, \eta)$
\end{proof}

\thmCombinedRegretSoftTcpaStrictBudget*
\begin{proof} Recall that $\tau$ is the stopping time of \cref{alg:combined-truthful-soft} as defined in \cref{defn:Defs}. Then, 
\begin{align*} 
\E_{\gseq\sim\calp^{T}}\left[\rew(\opt,\gseq,\rho)\right]&=\frac{\tau}{T}\cdot\E_{\gseq\sim\calp^{T}}\left[\rew(\opt,\gseq,\rho)\right]+\frac{T-\tau}{T}\cdot\E_{\gseq\sim\calp^{T}}\left[\rew(\opt,\gseq,\rho)\right]\\ 
        &\leq \frac{\tau}{T}\cdot\E_{\gseq\sim\calp^{T}}\left[\rew(\opt,\gseq,\rho)\right]+(T-\tau),\numberthis\label[ineq]{eq:6}
\end{align*} where the final step is because  $\rew(\opt,\gseq, \rho)\leq T$ (due to the value capped at one). Combining \cref{eq:6},  \cref{lem:weakdualityCombinedCase}, and the lower bound on $\rew(\cref{alg:combined-truthful-soft}, \gseq, \rho)$ from \cref{lem:minimumRegretCombinedCase} in the definition of $\reg$ in \cref{eq:defRegret} gives
\begin{align}
\reg(\cref{alg:combined-truthful-soft},\gseq)  &\leq\E_{\gseq\sim\calp^{T}}\left[(T-\tau)+\sum_{t=1}^{\tau}\mu_{t}\cdot(\rho-\pt(\bt))+\sum_{t=1}^{\tau}\lamt\cdot\gt(\bt)\right].\label[ineq]{eq:74}
\end{align}
We invoke \cref{prop:BoundingRhoTermMirrorDescent} with a total initial budget $\rho T$ and total time horizon $T$ to get \[\E_{\gseq \sim \calp^T} \left[ \sum_{t=1}^{\tau} \mu_t (\rho - p_t(b_t))\right] \leq \E_{\gseq \sim \calp^T}\left[\tau - T\right] + \frac{1}{\rho} + O(\sqrt{T}).\numberthis\label[ineq]{eq:CombinedAlgMutUpperBound}\] Finally,  we invoke \cref{lem:regret_bound} to conclude $\sum_{t=1}^R \lambda_t g_t \leq O(\sqrt{R})$, which lets us conclude \[\E_{\gseq\sim\calp^T} \left[\sum_{t=1}^{\tau} \lambda_t g_t\right] \leq O(\sqrt{T}).\numberthis\label[ineq]{eq:BoundFromTCPA}\] Combining \cref{eq:74},  \cref{eq:CombinedAlgMutUpperBound}, and  \cref{eq:BoundFromTCPA} yields the claimed  bound.
 
 We now finish with the proof of maximum constraint violation. By design of \cref{alg:combined-truthful-soft}, the budget constraint is never violated throughout the run of the algorithm. To see the claimed maximum violation of the RoS constraint, we note by \cref{prop:gradientProperties} that the gradient $g_t$ satisfies $g_t(b_t) \geq -1/\lambda_t$, as a result of which,  \cref{lem:constraint_violation} applies. 
\end{proof}

\subsection{Proofs for Strict RoS and Strict Budget Constraints}\label[app]{sec:strictCombinedAppendix}


\thmCombinedRegretStrictTcpaStrictBudget* 
\begin{proof}
The fact that the RoS constraint is not violated may be seen by the fact that the first phase accumulates the exact buffer that is the guaranteed cap on constraint violation by the second phase. The budget constraint is not violated by design: the first phase (which lasts at most $\rho T$ iterations) pays at most unit price per iteration, followed by the second phase, which, by the guarantee of \cref{alg:combined-truthful-soft}, strictly respects the budget constraint. 

To bound the regret, we note that the total expected reward  is at least as much as is collected in the second phase 
\begin{align*} 
\E_{\gseq \sim \calp^T} \left[ \rew(\cref{alg:combined-truthful-strict}, \gseq, \rho) \right]
&\geq 
\E_{k} \left[ \E_{\gseq_{\!k+1:T}\sim \calp^{T-k}} \left( \rew(\cref{alg:combined-truthful-soft}, \gseq_{\!k+1:T}, \widehat{\rho}\right) \right], \numberthis\label[ineq]{finalProof2}
\end{align*} where the notation on the right-hand side captures the reduced time horizon of $T-K(\gseq)$ as well as reduced initial budget of $\rho T - K(\gseq)$ for running \cref{alg:combined-truthful-soft} and $\widehat{\rho} = \frac{\rho T - K(\gseq))}{T-K(\gseq)}$. Conditioning on the event high-probability event that $k\leq \rho T$ (by \cref{eq:hpBoundOnK} coupled with the assumption that $\rho$ is a fixed constant), we have 
\begin{align*}  
& \E_{k} \left[ \E_{\gseq_{\!k+1:T}\sim \calp^{T-k}} \left( \rew(\cref{alg:combined-truthful-soft}, \gseq_{\!k+1:T}, \widehat{\rho}\right) \right] \\
&\quad \geq (1-e^{-O(T)}) \E_{k} \left[ \E_{\gseq_{\!k+1:T}\sim \calp^{T-k}} \left( \rew(\cref{alg:combined-truthful-soft}, \gseq_{\!k+1:T}, \widehat{\rho}\right) \mid k \leq \rho T \right]. \numberthis\label[ineq]{eq:conditionalExpKSmall}
\end{align*} Applying \cref{lem:minimumRegretCombinedCase} with the reduced budget and time horizon  gives: 
\begin{align*} 
\E_{\gseq_{\!k+1:T}\sim \calp^{T-k}}\left[\rew(\cref{alg:combined-truthful-soft},\gseq_{\!k+1:T}, \widehat{\rho}\right] &\geq\E_{\gseq_{\!k+1:T}\sim\calp^{T-k}} \left[\tau\cdot\bardcombined(\mub,\lamb|\calp, \widehat{\rho})\right]\\
&\quad -\E_{\gseq_{\!k+1:T} \sim \calp^{T-k}} \left[\sum_{t=1}^{\tau}\mu_{t}\cdot(\widehat{\rho}-\pt(\bt))\right]\\
&\quad- \E_{\gseq_{\!k+1:T} \sim \calp^{T-k}}\left[\sum_{t=1}^{\tau}\lamt\cdot g_t(b_t)\right], \numberthis\label[ineq]{eq:finalProof6}
\end{align*} Next, by \cref{lem:weakdualityCombinedCase}, we have: 
\[\bardcombined(\mub,\lamb|\calp, \widehat{\rho}) + (\rho-\widehat{\rho})\cdot \mub \geq \frac{1}{T}\E_{\gseq\sim\calp^{T}}\left[\rew(\opt,\gseq, \rho)\right]. \numberthis\label[ineq]{eq:NewIneq}\] 
We can now repeat the trick in the proof of \cref{lem:combinedRegret}: 
\begin{align*} 
\E_{\gseq\sim\calp^{T}}\left[\rew(\opt,\gseq,\rho)\right]
        &\leq \frac{\tau}{T}\cdot\E_{\gseq\sim\calp^{T}}\left[\rew(\opt,\gseq,\rho)\right]+(T-\tau),\numberthis\label[ineq]{eq:6Again}
\end{align*}
Combining \cref{finalProof2}, \cref{eq:conditionalExpKSmall}, \cref{eq:finalProof6}, \cref{eq:NewIneq}, and \cref{eq:6Again} then gives 
\begin{align*} 
\E_{\gseq \sim \calp^T} \left[ \reg(\cref{alg:combined-truthful-strict}, \gseq)\right] &\leq \E_{\gseq \sim \calp^T} (T-\tau) + \frac{T}{e^{O(T)}} \\ 
 &\quad +\E_k\left[ \E_{\gseq_{\!k+1:T} \sim \calp^{T-k}} \left[\sum_{t=1}^{\tau}\mu_{t}\cdot(\widehat{\rho}-\pt(\bt))\right] \mid k \leq \rho T\right]\\
&\quad+ \E_k\left[\E_{\gseq_{\!k+1:T} \sim \calp^{T-k}}\left[\sum_{t=1}^{\tau}\lamt\cdot g_t(b_t)\right] \mid k \leq \rho T\right]\\
&\quad+ \E_k\left[\E_{\gseq_{\!k+1:T} \sim \calp^{T-k}}\left[\tau \overline{\mu}_{\tau} (\rho - \widehat{\rho})\right] \mid k \leq \rho T\right]. \numberthis\label[ineq]{FinalProofRegretUpperBound1}
\end{align*} By applying \cref{prop:BoundingRhoTermMirrorDescent} and \cref{lem:FirstPhaseLength} and under the conditional expectation, we have 
\begin{align*} 
\E_k\left[ \E_{\gseq_{\!k+1:T} \sim \calp^{T-k}} \left[\sum_{t=1}^{\tau}\mu_{t}\cdot(\widehat{\rho}-\pt(\bt))\right] \mid k \leq \rho T\right]&\leq \E_{\gseq\sim \calp^T}\left[(\tau - T) + K(\gseq)\right] + O(\sqrt{T}) \\
& \quad + \E_{k}\left[\E_{\gseq_{\!k+1:T}}(1/\widehat{\rho})\mid k \leq \rho T\right]\\
&\leq \E_{\gseq\sim \calp^T} (\tau - T) + O(\sqrt{T}\log T).\numberthis\label[ineq]{FinalProofBound123} 
\end{align*} 
We invoke \cref{lem:regret_bound} to conclude $\sum_{t=1}^R \lambda_t g_t \leq O(\sqrt{R})$ for all $R\leq T$, which lets us conclude \[\E_{\gseq\sim\calp^T} \left[\sum_{t=1}^{\tau} \lambda_t g_t\right] \leq O(\sqrt{T}). \numberthis\label[ineq]{eq:FinalProofBoundLambdaTGt}\] To bound the final term in \cref{FinalProofRegretUpperBound1}, we observe that $\rho - \widehat{\rho} = \frac{(1-\rho)K(\gseq)}{T-K(\gseq)}$ by definition of $\widehat{\rho}$.  Combining this with $\tau \leq T$, the bound on $\sum_{i=1}^{\tau}\mu_i$ from \cref{alg:combined-truthful-soft}, the result of \cref{lem:FirstPhaseLength}, and the conditional expectation, we get  \[\E_k\left[\E_{\gseq_{\!k+1:T} \sim \calp^{T-k}}\left[\tau \overline{\mu}_{\tau} (\rho - \widehat{\rho})\right] \mid k \leq \rho T\right] \leq O(\sqrt{T}).\numberthis\label[ineq]{FinalProofNewBound}\]
Combining \cref{FinalProofRegretUpperBound1}, \cref{FinalProofBound123}, \cref{eq:FinalProofBoundLambdaTGt}, and \cref{FinalProofNewBound} finishes the proof. 
\end{proof}                                              

\section{Online Mirror Descent}\label[app]{sec:omd}
\begin{lem}[\cite{bubeck2015convex}, Theorem $4.2$]\label{lem:onlineMD}
Let $h$ be a mirror map which is $\rho$-strongly convex on $\mathcal{X}\cap \mathcal{D}$ with respect to a norm $\|{}\cdot{}\|$. Let $f$ be convex and $L$-Lipschitz with respect to $\|{}\cdot{}\|$. Then, mirror descent with step size $\alpha$ satisfies \[ \sum_{s=1}^t \left( f(x_s) - f(x) \right) \leq \frac{1}{\alpha} V_h(x, x_1) + \alpha \frac{L^2t}{2\rho}.\] 
\end{lem}

\begin{lem}[\cite{allen2014using}]
The Bregman divergence of the generalized negative entropy satisfies ``local strong convexity'': for any $x, y>0$, \[ V_h(y, x) = y \log(y/x) + x- y\geq \frac{1}{2\max(x,y)}\cdot (y-x)^2.\] 
\end{lem}
\begin{proof}
The claimed inequality is equivalent to \[t \log t  \geq (t-1) + \frac{1}{2\max(1,t)}\cdot (t-1)^2\numberthis\label[ineq]{eq:initLSCineq}\] for $ t >0$. 
Suppose $t\geq 1$. Then, choosing $u = 1-1/t$, \cref{eq:initLSCineq} is equivalent to 
\[-\log (1-u) \geq u + \frac{1}{2} u^2, \] for $u \in [0, 1)$, which holds by Taylor series. Suppose $0 < t \leq 1$. Then \cref{eq:initLSCineq}  is equivalent to \[ \log t - \frac{1}{2}\left( t - \frac{1}{t}\right)\geq 0,\] which may be checked by observing that the function is decreasing and equals zero at $t=1$. This completes the proof of the claim. 
\end{proof}
\end{appendices} 
\end{document}